\documentclass{article}

\usepackage{microtype}
\usepackage{graphicx}
\usepackage{subfigure}
\usepackage{booktabs} 
\usepackage{mathtools}
\usepackage{floatrow}
\usepackage{xr-hyper} 

\usepackage{hyperref}

\usepackage{tikz}
\newcommand\encircle[1]{%
  \tikz[baseline=(X.base)] 
    \node (X) [draw, shape=circle, inner sep=0] {\strut #1};}

\usepackage{url}

\usepackage{wrapfig}

\usepackage{enumitem}

\newcommand{\titleName}{A Fast and Scalable Joint Estimator for Integrating Additional Knowledge in Learning Multiple Related Sparse Gaussian Graphical Models}

\newcommand{\titleNameShort}{JEEK: Joint Elementary Estimator for Learning Multi-sGGM Using Additional Knowledge}

\usepackage{math-Header}

\newcommand{\sref}[1]{Section~(\ref{#1})} 
\newcommand{\eref}[1]{Eq.~(\ref{#1})} 
\newcommand{\rref}[1]{Theorem~(\ref{#1})} 
 
\newcommand{\lref}[1]{Lemma~(\ref{#1})} 
\newcommand{\cref}[1]{Condition~(\ref{#1})} 
 
\newcommand{\coref}[1]{Corollary~(\ref{#1})}

\usepackage{amsmath}
\usepackage{amssymb}
\usepackage{epsfig}

\def\R{{\mathbb R}}        

\def\P{{\mathbb P}}        
\def\E{{\mathbb E}}        
\def\1{{\mathbf 1}}        

\def\x{{\mathbf x}}

\usepackage{subfigure} 

\usepackage{algorithm}
\usepackage{algorithmic}

\usepackage{lipsum}
\usepackage{titlesec}

\titlespacing\section{0pt}{3pt plus 0pt minus 1pt}{0pt plus 1pt minus 1pt}
\titlespacing\subsection{0pt}{1pt plus 0pt minus 1pt}{0pt plus 0pt minus 0pt}
\titlespacing\subsubsection{0pt}{1pt plus 0pt minus 1pt}{0pt plus 0pt minus 0pt}
\titlespacing{\paragraph}{0pt}{1pt}{0pt}[0pt]  
\usepackage{etoolbox}
\makeatletter
\preto{\@tabular}{\parskip=5pt}
\makeatother

\allowdisplaybreaks

\usepackage{enumitem}
\setlist[itemize]{leftmargin=*}

\usepackage{color}
\definecolor{colora}{rgb}{.7, .1, .1}

\usepackage[accepted]{icml2018}

\icmltitlerunning{\titleNameShort}

\begin{document}

\twocolumn[
\icmltitle{\titleName}



\icmlsetsymbol{equal}{*}

\begin{icmlauthorlist}
\icmlauthor{Beilun Wang}{uva}
\icmlauthor{Arshdeep Sekhon}{uva}
\icmlauthor{Yanjun Qi}{uva}
\end{icmlauthorlist}

\icmlaffiliation{uva}{Department of Computer Science, University of Virginia, 
\emph{http://www.jointnets.org/}}

\icmlcorrespondingauthor{Beilun Wang}{bw4mw@virginia.edu}
\icmlcorrespondingauthor{Yanjun Qi}{yanjun@virginia.edu}

\icmlkeywords{Machine Learning, ICML}

\vskip 0.3in
]



\printAffiliationsAndNotice{}  

\begin{abstract}

We consider the problem of including additional knowledge in estimating sparse Gaussian graphical models (sGGMs) from aggregated samples, arising often in bioinformatics and neuroimaging applications. Previous joint sGGM estimators either fail to use existing knowledge or cannot scale-up to many tasks (large $K$) under a high-dimensional (large $p$) situation.  In this paper, we propose a  novel \underline{J}oint \underline{E}lementary \underline{E}stimator incorporating additional \underline{K}nowledge (JEEK) to infer multiple related sparse Gaussian Graphical models from large-scale heterogeneous data. Using domain knowledge as weights, we design a novel hybrid norm as the minimization objective to enforce the superposition of two weighted sparsity constraints, one on the shared interactions and the other on the task-specific structural patterns. This enables JEEK to elegantly consider various forms of existing knowledge based on the domain at hand and avoid the need to design knowledge-specific optimization. JEEK is solved through a fast and entry-wise parallelizable solution that largely improves the computational efficiency of the state-of-the-art  $O(p^5K^4)$ to $O(p^2K^4)$. We conduct a rigorous statistical analysis showing that JEEK achieves the same  convergence rate $O(\log(Kp)/n_{tot})$ as the state-of-the-art estimators that are much harder to compute. 
Empirically, on multiple synthetic datasets and two real-world data, JEEK outperforms the speed of the state-of-arts significantly while achieving the same level of prediction accuracy. \footnote{In this updated version, we correct one equation error we had before about kw norm's dual form.  Our code implementation was correct, therefore no change in our toolbox and empirical results.}

\end{abstract}

\section{Introduction}
\label{sec:intro}

Technology revolutions in the past decade have collected large-scale heterogeneous samples from many scientific domains. For instance, genomic technologies have delivered petabytes of molecular measurements across more than hundreds of types of cells and tissues from national projects like ENCODE~\cite{encode2012integrated} and TCGA~\cite{cancer2011integrated}. Neuroimaging technologies have generated petabytes of functional magnetic resonance imaging (fMRI) datasets across thousands of human subjects (shared publicly through projects like openfMRI~\cite{poldrack2013toward}. Given such data, understanding and quantifying variable graphs from heterogeneous samples about multiple contexts is a fundamental analysis task.

Such variable graphs can significantly simplify network-driven studies about diseases \cite{ideker2012differential}, can help understand the neural characteristics underlying clinical disorders \cite{uddin2013salience} and can allow for understanding genetic or neural pathways and systems. 
The number of contexts (denoted as $K$) that those applications need to consider grows extremely fast, ranging from tens (e.g., cancer types in TCGA~\cite{cancer2011integrated}) to thousands (e.g., number of subjects in openfMRI~\cite{poldrack2013toward}). The number of variables (denoted as $p$) ranges from hundreds (e.g., number of brain regions) to tens of thousands (e.g., number of human genes).

The above data analysis problem can be formulated as jointly estimating $K$ conditional dependency graphs $G^{(1)}, G^{(2)},\dots,G^{(K)}$ on a single set of $p$ variables based on heterogeneous samples accumulated from $K$ distinct contexts. For homogeneous data samples from a given $i$-th context, one typical approach is the sparse Gaussian Graphical Model (sGGM)~\cite{lauritzen1996graphical,yuan2007model}. sGGM assumes  samples are independently and identically drawn from $N_p(\mu^{(i)}, \Sigma^{(i)})$, a multivariate
Gaussian distribution with  mean vector $\mu^{(i)}$ and covariance matrix $\Sigma^{(i)}$. 
The graph structure $G^{(i)}$ is encoded by the sparsity
pattern of the inverse covariance matrix, also named precision
matrix, $\Omega^{(i)}$. $\Omega^{(i)} := ({\Sigma^{(i)}})^{-1}$. $\Omega^{(i)}_{jk} = 0$ if and only if in $G^{(i)}$ an edge does not connect $j$-th node and $k$-th node (i.e., conditional independent). sGGM imposes an $\ell_1$ penalty on the parameter $\Omega^{(i)}$ to achieve a consistent estimation under high-dimensional situations.
When handling heterogeneous data samples, rather than estimating sGGM of each condition separately,
a multi-task formulation that jointly estimates $K$ different but related sGGMs can lead to a better generalization\cite{caruana1997multitask}.

Previous studies for joint estimation of multiple sGGMs roughly fall into four categories:  \cite{danaher2013joint,mohan2013node,chiquet2011inferring,honorio2010multi,guo2011joint,zhang2012learning,zhang2010learning,zhu2014structural}:
(1) The first group seeks to optimize a sparsity regularized data likelihood function plus an extra penalty function $\mathcal{R}'$ to enforce structural similarity among multiple estimated networks. Joint graphical lasso (JGL)
\cite{danaher2013joint} proposed an alternating direction method of multipliers (ADMM) based optimization algorithm to work with two regularization functions ($\ell_1 + \mathcal{R}'$). (2) The second category tries to recover the support of $\Omega^{(i)}$ using sparsity penalized regressions in a column by column fashion. Recently \cite{monti2015learning} proposed to learn population and subject-specific brain connectivity networks via a so-called ``Mixed Neighborhood Selection'' (MSN) method in this category. (3) The third type of methods seeks to minimize the joint sparsity of the target precision matrices under matrix inversion constraints. One recent study, named SIMULE (Shared and Individual parts of MULtiple graphs Explicitly) \cite{wang2017constrained}, automatically infers both specific edge patterns that are unique to each context and shared interactions preserved among all the contexts (i.e. by modeling each precision matrix as $\Omega^{(i)} = \Omega^{(i)}_I + \Omega_S$) via the constrained $\ell_1$ minimization. Following the CLIME estimator \cite{JMLR:v15:pang14a}, the constrained $\ell_1$ convex formulation can also be solved column by column via linear programming. However, all three categories of aforementioned estimators are difficult to scale up when the dimension $p$ or the number of tasks $K$ are large because they cannot avoid expensive steps like SVD \cite{danaher2013joint} for JGL, linear programming for SIMULE or running multiple iterations of $p$ expensive penalized regressions in MNS. (4) The last category extends the so-called "Elementary Estimator" graphical model (EE-GM) formulation \cite{yang2014elementary} to revise JGL's penalized likelihood into a constrained convex program that minimizes ($\ell_1 + \mathcal{R}'$).  One proposed estimator FASJEM \cite{wang2017fast} is solved in an entry-wise manner and group-entry-wise manner that largely outperforms the speed of its JGL counterparts.  More details of the related works are in \sref{sec:rel}.

One significant caveat of state-of-the-art joint sGGM estimators is the fact that little attention has been paid to incorporating existing knowledge of the nodes or knowledge of the relationships among nodes in the models. 
In addition to the samples themselves, additional information is widely available in real-world applications. In fact, incorporating the knowledge is of great scientific interest. A prime example is when estimating the functional brain connectivity networks among brain regions based on fMRI samples, the spatial position of the regions are readily available. Neuroscientists have gathered considerable knowledge regarding the spatial and anatomical evidence underlying brain connectivity (e.g., short edges and certain anatomical regions are more likely to be connected \cite{watts1998collective}). Another important example is the problem of identifying gene-gene interactions from patients' gene expression profiles across multiple cancer types. Learning the statistical dependencies among genes from such heterogeneous datasets can help to understand how such dependencies vary from normal to abnormal and help to discover contributing markers that influence or cause the diseases. Besides the patient samples, state-of-the-art bio-databases like HPRD \cite{prasad2009human} have collected a significant amount of information about direct physical interactions among corresponding proteins, regulatory gene pairs or signaling relationships collected from high-qualify bio-experiments.

Although being strong evidence of structural patterns we aim to discover, this type of information has rarely been considered in the joint sGGM formulation of such samples. To the authors' best knowledge, only one study named as W-SIMULE tried to extend the constrained $\ell_1$ minimization in SIMULE into weighted $\ell_1$ for considering spatial information of brain regions in the joint discovery of heterogeneous neural connectivity graphs \cite{singh2017constrained}. This method was designed just for the neuroimaging samples and has $O(p^5K^4)$ time cost, making it not scalable for large-scale settings (more details in Section~\ref{sec:meth}).

This paper aims to fill this gap by adding additional knowledge most effectively into scalable and fast joint sGGM estimations. We propose a novel model, namely \underline{J}oint \underline{E}lementary \underline{E}stimator incorporating additional \underline{K}nowledge (JEEK), that presents a principled and scalable strategy to include additional knowledge when estimating multiple related sGGMs jointly.  Briefly speaking, this paper makes the following contributions:

    \setlist{nolistsep}
\begin{itemize}[noitemsep]
    \item \textbf{Novel approach:} JEEK presents a new way of integrating additional knowledge in learning multi-task sGGMs in a scalable way. (Section~\ref{sec:meth})
    \item \textbf{Fast optimization:} 
  We optimize JEEK through an entry-wise and group-entry-wise manner that can dramatically improve the time complexity to $O(p^2K^4)$. (Section~\ref{sec:optm})
    \item \textbf{Convergence rate:} We theoretically prove the convergence rate of JEEK as $O(\log(Kp)/n_{tot})$. This rate shows the benefit of joint estimation and achieves the same convergence rate as the state-of-the-art that are much harder to compute. 
    (Section~\ref{sec:theory})
    \item \textbf{Evaluation:} We evaluate JEEK using several synthetic datasets and two real-world data, one from neuroscience and one from genomics. It outperforms state-of-the-art baselines significantly regarding the speed. (Section~\ref{sec:exp})
\end{itemize}

JEEK provides the flexibility of using ($K+1$) different weight matrices representing the extra knowledge.
We try to showcase a few possible designs of the weight matrices in Section~\ref{sec:designW}, including (but not limited to): 
\begin{itemize}
\item Spatial or anatomy knowledge about brain regions;
\item Knowledge of known co-hub nodes or perturbed  nodes; 
\item Known group information about nodes, such as genes belonging to the same biological pathway or cellular location; 
\item Using existing known edges as the knowledge, like the known protein interaction databases for discovering gene networks (a semi-supervised setting for such estimations).
\end{itemize}
We sincerely believe the scalability and flexibility provided by JEEK can make structure learning of joint sGGM feasible in many real-world tasks.

\textit{Att:} Due to space limitations, we have put details of certain contents (e.g., proofs)  in the appendix. Notations with ``S:'' as the prefix in the numbering mean the corresponding contents are in the appendix. For example, full proofs are in \sref{sec:proof}.

\textit{Notations: } math notations we use are described in \sref{sec:moremeth}. $n_{tot} =\sum\limits_{i=1}^{K}n_i$ is the total number of data samples. The Hadamard product $\circ$ is the element-wise product between two matrices. Also to simplify the notations, we abuse the notation $\dfrac{1}{W}$ to represent a new matrix being generated by element wise division of each entry in $W$ by $1$.

\section{Background}
\label{sec:background}

\paragraph{Sparse Gaussian graphical model (sGGM): }
The classic formulation of estimating sparse Gaussian Graphical model \cite{yuan2007model} from a single given condition (single sGGM) is the ``graphical lasso'' estimator (GLasso) \cite{yuan2007model,banerjee2008model}. It solves the following $\ell_1$ penalized maximum likelihood estimation (MLE) problem:
\begin{equation}
  \label{eq:ggm}
  \argmin\limits_{\Omega>0} -\log\text{det}(\Omega) + <\Omega,\hat{\Sigma}> + \lambda_n ||\Omega||_1
\end{equation}

\paragraph{M-Estimator with Decomposable Regularizer in High-Dimensional Situations:}~Recently the seminal study~\cite{negahban2009unified} proposed a unified framework for high-dimensional analysis of the following general formulation: M-estimators with decomposable regularizers:
\begin{equation}
\label{eq:m-estimator}
  \argmin\limits_{\theta} \mathcal{L}(\theta) + \lambda_n \mathcal{R}(\theta)
\end{equation}
where $\mathcal{R}(\cdot)$ represents a decomposable regularization function and $\mathcal{L}(\cdot)$ represents a loss function (e.g., the negative log-likelihood function in sGGM $\mathcal{L}(\Omega)=-\log\text{det}(\Omega) + <\Omega,\hat{\Sigma}>$). Here $\lambda_n > 0$ is the tuning parameter. 
\\

\paragraph{Elementary Estimators (EE):}~
Using the analysis framework from ~\cite{negahban2009unified}, recent studies~\cite{yang2014elementary1,yang2014elementary,yang2014elementary2} propose a new category of estimators named ``Elementary estimator'' (EE) with the following general formulation: 
\begin{equation}
\label{eq:ee}
  \begin{split}
    &\argmin\limits_{\theta} \mathcal{R}(\theta)\\
    &\text{subject to:} \mathcal{R}^*(\theta -\hat{\theta}_n) \le \lambda_n 
    \end{split}
\end{equation}
Where $\mathcal{R}^*(\cdot)$ is the dual norm of $\mathcal{R}(\cdot)$,  
\begin{equation}
\mathcal{R}^*(v) := \sup\limits_{u \ne 0}\frac{<u,v>}{\mathcal{R}(u)} = \sup\limits_{\mathcal{R}(u) \le 1}<u,v>.
\end{equation}

The solution of \eref{eq:ee} achieves the near optimal convergence rate as \eref{eq:m-estimator} when satisfying certain conditions.
$\mathcal{R}(\cdot)$ represents a decomposable regularization function (e.g., $\ell_1$-norm) and $\mathcal{R}^*(\cdot)$ is the dual norm of $\mathcal{R}(\cdot)$ (e.g., $\ell_{\infty}$-norm is the dual norm of $\ell_1$-norm). $\lambda_n$ is a regularization parameter.

The basic motivation of \eref{eq:ee} is to build simpler and possibly fast estimators, that yet come with statistical guarantees that are nonetheless comparable to regularized MLE. $\hat{\theta}_n$  needs to be carefully constructed, well-defined and closed-form for the purpose of simpler computations.  The formulation defined by \eref{eq:ee} is to ensure its solution having the desired structure defined by $\mathcal{R}(\cdot)$. For cases of high-dimensional estimation of linear regression models, $\hat{\theta}_n$ can be the classical ridge estimator that itself is closed-form and with strong statistical convergence guarantees in high-dimensional situations.

\paragraph{EE-sGGM:}  \cite{yang2014elementary} proposed elementary estimators for graphical models (GM) of exponential families, in which 
$\hat{\theta}_n$ represents so-called proxy of backward mapping for the target GM (more details in Section~\ref{seca:backward}). 
The key idea  (summarized in the upper row of Figure~\ref{fig:digram}) is to investigate the vanilla MLE and where it “breaks down” for estimating a graphical model of exponential families in the case of high-dimensions \cite{yang2014elementary}. Essentially the vanilla graphical model MLE can be expressed as a backward mapping that computes the model parameters from some given moments in an exponential family distribution. For instance, in the case of learning Gaussian GM (GGM) with vanilla MLE, the backward mapping is $\hat{\Sigma}^{-1}$ that estimates $\Omega$ from the sample covariance matrix (moment) $\hat{\Sigma}$. We introduce the details of backward mapping in  Section~\ref{seca:backward}.

However, even though this backward mapping has a simple closed form for GGM, the backward mapping is normally not well-defined in high-dimensional settings. When given the sample covariance $\hat{\Sigma}$, we cannot just compute the vanilla MLE solution as $[\hat{\Sigma}]^{-1}$ for GGM since $\hat{\Sigma}$ is rank-deficient when $p>n$. Therefore Yang et al. \cite{yang2014elementary} used carefully constructed proxy backward maps as $\hat{\theta}_n = [T_v(\hat{\Sigma})]^{-1}$  that is both available in closed-form, and well-defined in high-dimensional settings for GGMs. 
We introduce the details of $[T_v(\hat{\Sigma})]^{-1}$ and its statistical property in  Section~\ref{seca:backward}. Now \eref{eq:ee} becomes the following closed-form estimator for learning sparse Gaussian graphical models ~\cite{yang2014elementary}: 
\begin{equation}
  \begin{split}
    \argmin_{\Omega} & ||\Omega||_{1,,\text{off}}\\
    \text{subject to:} &||\Omega - [T_v(\hat{\Sigma})]^{-1}||_{\infty,\text{off}} \le \lambda_n
    \label{eq:eeggm}
    \end{split}
    \vspace{-5mm}
\end{equation}
\eref{eq:eeggm} is a special case of \eref{eq:ee}, in which $\mathcal{R}(\cdot)$ is the off-diagonal $\ell_1$-norm and the precision matrix $\Omega$ is the $\theta$ we search for. 
When $\mathcal{R}(\cdot)$ is the  $\ell_1$-norm, the solution of \eref{eq:ee} (and \eref{eq:eeggm}) just needs to  perform entry-wise thresholding operations on $\hat{\theta}_n$  to ensure the desired sparsity structure of its final solution.

\begin{figure}
    \centering
    \includegraphics[width=\linewidth]{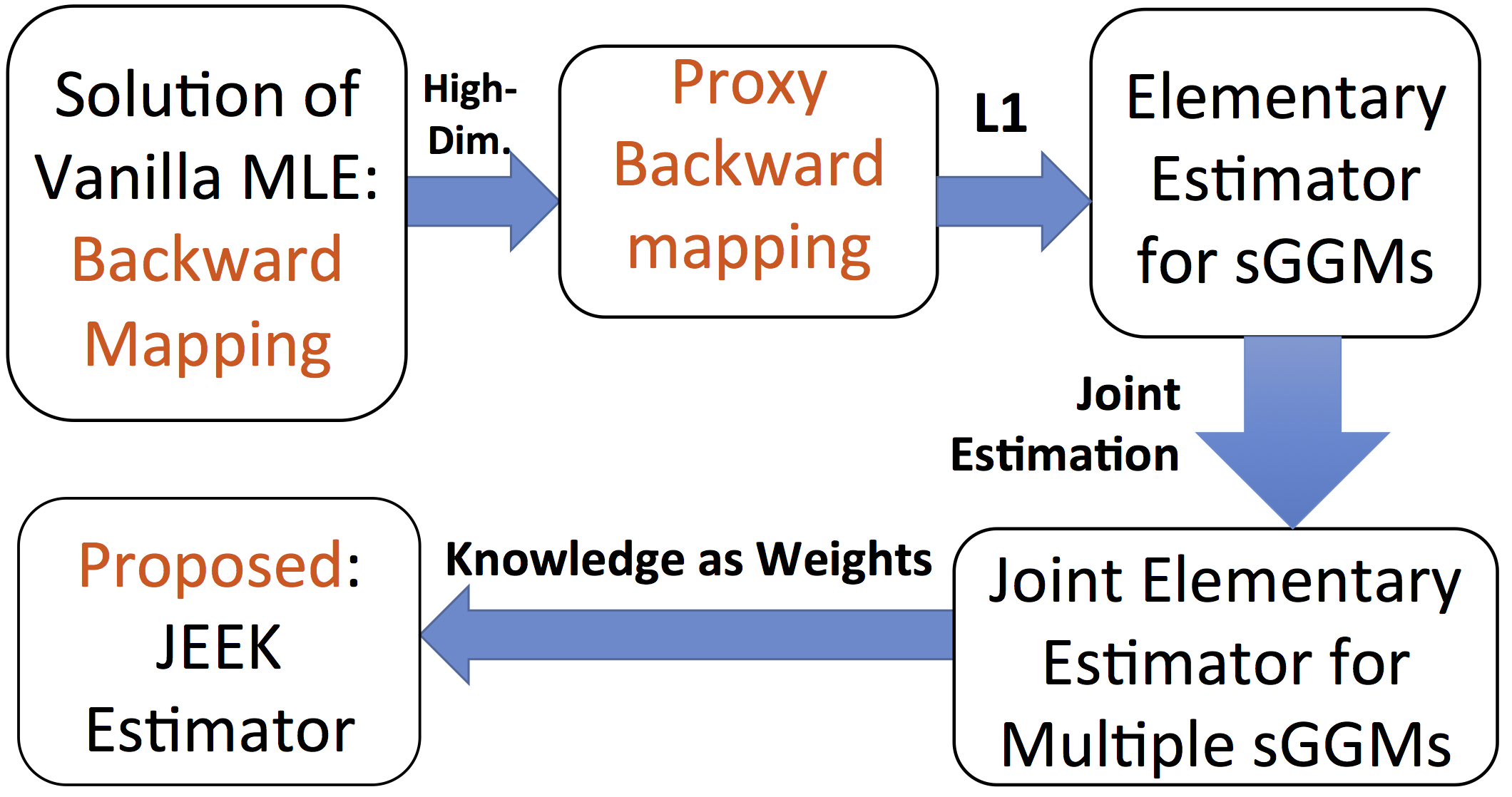}
\label{fig:digram}
\caption{Basic idea of JEEK.}
\vspace{-5mm}
\end{figure}

\section{Proposed Method: JEEK}
\label{sec:meth}

In applications of Gaussian graphical models, we typically have more information than just the data samples themselves. 
This paper aims to propose a simple, scalable and theoretically-guaranteed joint estimator for estimating multiple sGGMs with additional knowledge in large-scale situations.

\subsection{A Joint EE (JEE) Formulation}
We first propose to jointly estimate multiple related sGGMs from $K$ data blocks using the following  formulation:
\begin{equation}
\label{eq:sumsggm}
\begin{split}
\argmin\limits_{\Omega^{(1)}, \Omega^{(2)},
\dots,\Omega^{(K)}} &\sum\limits_{i = 1}^{K} \mathcal{L}(\Omega^{(i)}) + 
\lambda_n  \mathcal{R}(\Omega^{(1)}, \Omega^{(2)},
\dots,\Omega^{(K)})
\end{split}
\end{equation}
where $\Omega^{(i)}$ denotes the precision matrix for $i$-th task. $\mathcal{L}(\Omega)=-\log\text{det}(\Omega) + <\Omega,\hat{\Sigma}>$ describes the negative log-likelihood function in sGGM. $\Omega^{(i)} \succ 0$ means that $\Omega^{(i)}$ needs to be a positive
definite matrix.  $\mathcal{R}(\cdot)$ represents a decomposable regularization function enforcing sparsity and structure assumptions (details in \sref{sec:wkn}).

For ease of notation,  we denote that $\Omega^{tot} = (\Omega^{(1)},\Omega^{(2)},\dots,\Omega^{(K)})$ and $\Sigma^{tot} = (\Sigma^{(1)},\Sigma^{(2)},\dots,\Sigma^{(K)})$. $\Omega^{tot}$ and $\Sigma^{tot}$ are both $p \times Kp$ matrices (i.e., $Kp^2$ parameters to estimate). Now define an inverse function as $\text{inv}(A^{tot}) := ({A^{(1)}}^{-1},{A^{(2)}}^{-1},\dots,{A^{(K)}}^{-1})$, where $A_{tot}$ is a given $p \times Kp$ matrix with the same structure as $\Sigma_{tot}$. Then we rewrite  \eref{eq:sumsggm} into the following form: 
\begin{equation}
\label{eq:sumsggmu}
\begin{split}
\argmin\limits_{\Omega^{tot} } & \mathcal{L}(\Omega^{tot}) + 
\lambda_n  \mathcal{R}(\Omega^{tot})
\end{split}
\end{equation}

Now  connecting \eref{eq:sumsggmu} to \eref{eq:m-estimator} and \eref{eq:ee}, we propose the following joint elementary estimator (JEE) for learning multiple sGGMs: 
\begin{equation}
\label{eq:eetot}
  \begin{split}
    &\argmin\limits_{\Omega^{tot}} \mathcal{R}(\Omega^{tot})\\
    &\text{subject to: } \mathcal{R}^*(\Omega^{tot} -\hat{\Omega}^{tot}_{n_{tot}}) \le \lambda_n 
    \end{split}
\end{equation}

 The fundamental component in~\eref{eq:ee} for the single context sGGM was to use a well-defined proxy function  to approximate the vanilla MLE solution (named as the backward mapping for exponential family distributions)~\cite{yang2014elementary}. The proposed proxy $\hat{\theta}_n=[T_v(\hat{\Sigma})]^{-1}$ is both well-defined under high-dimensional situations and also has a simple closed-form. Following a similar idea, when learning multiple sGGMs, we propose to use $inv(T_v(\hat{\Sigma}^{tot}))$ for $\hat{\Omega}^{tot}_{n_{tot}}$ and get the following joint elementary estimator:  
\begin{equation}
\begin{split}
    \label{eq:JEEK-gen}
     \argmin\limits_{\Omega^{tot}} & \mathcal{R}(\Omega^{tot}) \\
    \text{Subject to: } & \mathcal{R}^*( \Omega^{tot} - inv(T_v(\hat{\Sigma}^{tot})) ) \le \lambda_n
\end{split}
\end{equation}

\subsection{Knowledge as Weight (KW-Norm)} 
\label{sec:wkn}

The main goal of this paper is to design a principled strategy to incorporate existing knowledge (other than samples or structured assumptions) into the multi-sGGM formulation. We consider two factors in such a design: 

(1) When learning multiple sGGMs jointly from real-world applications, it is often of great scientific interests to model and learn context-specific graph variations explicitly, because such variations can ``fingerprint'' important markers in domains like cognition \cite{ideker2012differential} or pathology \cite{kelly2012characterizing}. Therefore we design to share parameters between different contexts. Mathematically, we model $\Omega^{(i)}$ as two parts:
\begin{equation}
\label{eq:omega_i}
\Omega^{(i)} = \Omega_I^{(i)} + \Omega_S
\end{equation}
where $\Omega_I^{(i)}$ is the individual precision matrix for context $i$ and $\Omega_S$ is the shared precision matrix between contexts. Again, for ease of notation we denote $\Omega_I^{tot} = (\Omega_I^{(1)},\Omega_I^{(2)},\dots,\Omega_I^{(K)})$ and $\Omega_S^{tot}=(\Omega_S,\Omega_S,\dots,\Omega_S)$.

(2) We represent additional knowledge as positive weight matrices from $\R^{p \times p}$. More specifically, we represent the knowledge of the task-specific graph as weight matrix $\{W^{(i)}\}$ and $W_S$ representing existing knowledge of the shared network. The positive matrix-based representation is a powerful and flexible strategy that can describe many possible forms of existing knowledge. In \sref{sec:designW}, we provide a few different designs of $\{W^{(i)}\}$ and $W_S$ for real-world applications. In total, we have weight matrices $\{ W^{(1)}_I, W^{(2)}_I, \dots, W^{(K)}_I, W_S \}$ to represent additional knowledge. To simplify notations, we denote $W^{tot}_I = (W^{(1)}_I, W^{(2)}, \dots,  W^{(K)}_I)$ and $W^{tot}_S = (W_S, W_S,\dots, W_S)$.

Now we propose the following \underline{k}nowledge as \underline{w}eight norm (kw-norm) combining the above two: 
\begin{equation} 
    \label{eq:kw-norm}
    \mathcal{R}(\Omega^{tot}) = ||W^{tot}_I\circ\Omega^{tot}_I||_1+ ||W^{tot}_S \circ \Omega^{tot}_S||_1
\end{equation}
Here the Hadamard product $\circ$ is the element-wise product between two matrices i.e. $[A\circ B]_{ij} = A_{ij}B_{ij}$. 

The kw-norm(~\eref{eq:kw-norm}) has the following three properties:
\begin{itemize}
    \item (i) kw-norm is a norm function if and only if any entries in $W^{tot}_I$ and $W^{tot}_S$ do not equal to $0$.
    \item (ii) If the condition in (i) holds, kw-norm is a decomposable norm.
    \item (iii) If the condition (i) holds, the dual norm of kw-norm is $\mathcal{R}^*(u) = \max(|| \dfrac{1}{W^{tot}_I}  \circ     u||_{\infty}, ||\dfrac{1}{W^{tot}_S}  \circ u||_{\infty})$ \footnote{In our previous version, we mistakenly wrote $\mathcal{R}^*(u) = \max(|| W^{tot}_I  \circ     u||_{\infty}, ||W^{tot}_S  \circ u||_{\infty})$. We correct relevant equations here. Our code implementation was correct, thus not change.\label{xx}} 
\end{itemize}
Section~\ref{sec:proofkw} provides proofs of the above claims.

\subsection{JEE with Knowledge (JEEK)} 

Plugging \eref{eq:kw-norm} to \eref{eq:JEEK-gen}, we obtain the following formulation of JEEK for learning multiple related sGGMs from heterogeneous samples:
\begin{equation}
\begin{split}
    \label{eq:JEEK}
     \argmin\limits_{\Omega^{tot}_I, \Omega^{tot}_S} & ||W^{tot}_I \circ \Omega^{tot}_I||_1 + ||W^{tot}_S\circ \Omega^{tot}_S|| \\
    \text{\textsuperscript{\ref{xx}} Subject to: } & ||\dfrac{1}{W^{tot}_I} \circ (\Omega^{tot} - inv(T_v(\hat{\Sigma}^{tot}))) ||_{\infty} \le \lambda_n\\
    & ||\dfrac{1}{W^{tot}_S} \circ (\Omega^{tot} - inv(T_v(\hat{\Sigma}^{tot}))) ||_{\infty} \le \lambda_n \\
    & \Omega^{tot} = \Omega^{tot}_S + \Omega^{tot}_I 
\end{split}
\end{equation}

In Section~\ref{sec:theory}, we theoretically prove that the statistical convergence rate of JEEK achieves the same sharp convergence rate as the state-of-the-art estimators for multi-task sGGMs. Our proofs are inspired by the unified framework of the high-dimensional statistics~\cite{negahban2009unified}.

\subsection{Solution of JEEK:}~
\label{sec:optm}

A huge computational advantage of JEEK (\eref{eq:JEEK}) is that it can be decomposed into $p \times p$ independent small linear programming problems. To simplify notations, we denote ${\Omega^{(i)}_{I}}_{j,k}$ (the $\{j,k\}$-th entry of $\Omega^{(i)}$) as $a_i$. Similarly we denote  ${\Omega_{S}}_{j,k}$ as $b$ and $[T_v(\hat{\Sigma}^{(i)})]^{-1}_{j,k}$ be $c_i$. Similarly we denote $W^{(i)}_{j,k} = w_i$ and $W^S_{j,k} = w_s$. "A group of entries" means a set of parameters $\{a_1,\dots, a_K, b\}$ for  certain $j,k$.

In order to estimate $\{a_1,\dots, a_K, b\}$, JEEK (\eref{eq:JEEK}) can be decomposed into the following formulation for a certain $j,k$  :
\begin{equation}
\label{eq:JEEK-par}
    \begin{split}
     \argmin\limits_{a_i,b} &\sum\limits_i |w_i a_i| + K|w_s b|  \\
    \text{ \textsuperscript{\ref{xx}} Subject to:} \; & |a_i+b - c_i| \le \lambda_{n}\min(w_i,w_s), \\
    & i = 1,\dots,K 
    \end{split}
\end{equation}

\eref{eq:JEEK-par} can be easily converted into a linear programming form of \eref{eq:JEEK-lp} with only $K+1$ variables. The time complexity of~\eref{eq:JEEK-par} is $O(K^4)$. Considering JEEK has a total $p(p-1)/2$ of such subproblems to solve, the computational complexity of JEEK (\eref{eq:JEEK}) is therefore $O(p^2K^4)$. 
We summarize the optimization algorithm of JEEK in Algorithm~\ref{alg:JEEK} (details in \sref{sec:parallel}).

\begin{algorithm}[t]
\vspace{-1mm}
   \caption{Joint Elementary Estimator with additional knowledge (JEEK) for Multi-task sGGMs}
   \label{alg:JEEK}
{\footnotesize
\begin{flushleft}
 \textbf{Input:} Data sample matrix $\Xb^{(i)}$ ( $i = 1$ {\bfseries to} $K$), regularization hyperparameter 
  $\lambda_n$, Knowledge weight matrices $\{ W_I^{(i)}, W_S\}$ and \textbf{LP(.)} (a linear programming solver)\\
 \textbf{Output:} $\{ \Omega^{(i)} \}$  ( $i = 1$ {\bfseries to} $K$)
\end{flushleft}
\begin{algorithmic}[1]
    \FOR{$i = 1$ {\bfseries to} $K$}
        \STATE Initialize $\hat{\Sigma}^{(i)} = \frac{1}{n_i-1} \sum_{s=1}^{n_i} (\Xb^{(i)}_{s,}-\hat{\mu}^{(i)}) 
        (\Xb^{(i)}_{s,} -\hat{\mu}^{(i)})^T$ (the sample covariance matrix of $\Xb^{(i)}$)
    \STATE Initialize $\Omega^{(i)} = {\bf 0}_{p\times p}$
    \STATE Calculate the proxy backward mapping $[T_v(\hat{\Sigma}^{(i)})]^{-1}$
    \ENDFOR
   \FOR{$j=1$ {\bfseries to} $p$}
   \FOR{$k=1$ {\bfseries to} $j$}
    \STATE $c_i = [T_v(\hat{\Sigma}^{(i)})]^{-1}_{j,k}$
    \STATE $w_i = W^{(i)}_{j,k}$
    \STATE $w_s = {W_S}_{j,k}$
    \STATE $a_i,b = \textbf{LP}(w_i, w_s, c_i, \lambda_n)$ where $i = 1,\dots,K$ and \textbf{LP(.)} solves \eref{eq:JEEK-par}
    \FOR {$i=1$ {\bfseries to} $K$}
    \STATE ${\Omega^{(i)}}_{j,k} = {\Omega^{(i)}}_{k,j} = a_i + b$
    \STATE ${\Omega^{(i)}_{I}}_{j,k} =a_i$
    \STATE ${\Omega_{S}}_{j,k} = b$
    \ENDFOR
    \ENDFOR
    \ENDFOR
\end{algorithmic}
}
\vspace{-1mm}
\end{algorithm}
\section{Theoretical Analysis}
\label{sec:theory}

\paragraph{KW-Norm:} We presented the three properties of kw-norm in Section~\ref{sec:wkn}. 
The proofs of these three properties are included in \sref{sec:proofkw}.

\paragraph{Theoretical error bounds of Proxy Backward Mapping: }  \cite{yang2014elementary} proved that when ($p\ge n$), the proxy backward mapping $[T_v(\hat{\Sigma})]^{-1}$ used 
by EE-sGGM achieves the sharp convergence rate to its truth (i.e., by proving $||T_v(\hat{\Sigma}))^{-1} - {\Sigma^*}^{-1}||_{\infty} = O(\sqrt{\frac{\log p}{n}})$). The proof was extended from the previous study~\cite{rothman2009generalized} that devised $T_v(\hat{\Sigma})$ for estimating covariance matrix consistently in high-dimensional situations. See detailed proofs in  Section~\ref{sec:proofbm}. To derive the statistical error bound of JEEK, we need to assume that $inv(T_v(\hat{\Sigma}^{tot}))$ are well-defined. This is ensured by  
assuming that the true ${\Omega^{(i)}}^*$ satisfy the conditions defined in \sref{sec:proofkw}.  

\paragraph{Theoretical error bounds of JEEK: } 
We now use the high-dimensional analysis framework from \cite{negahban2009unified},  three properties of kw-norm, and error bounds of backward mapping from \cite{rothman2009generalized,yang2014elementary} to derive the statistical convergence rates of JEEK. Detailed proofs of the following theorems are in  Section~\ref{sec:theory} .

Before providing the theorem, we need to define the structural assumption, the IS-Sparsity, we assume for the parameter truth. \\
\textbf{(IS-Sparsity):} The 'true' parameter of  ${\Omega^{tot}}^*$ can be decomposed into two clear structures--$\{ {\Omega^{tot}_I}^*$ and ${\Omega^{tot}_S}^* \}$. ${\Omega^{tot}_I}^*$ is exactly sparse with $k_i$ non-zero entries indexed by a support set $S_I$ and ${\Omega^{tot}_S}^*$ is exactly sparse with $k_s$ non-zero entries indexed by a support set $S_S$. $S_I\bigcap S_S = \emptyset$. All other elements  equal to $0$ (in $(S_I\bigcup S_S)^c$). \\

\begin{theorem}
\vspace{-3mm}
\label{theo:jeek}
Consider  $\Omega^{tot}$ whose true parameter ${\Omega^{tot}}^*$ satisfies the \textbf{(IS-Sparsity)} assumption. Suppose we compute the solution of~\eref{eq:JEEK} with a bounded $\lambda_n$ such that $\lambda_n \ge \max(||\dfrac{1}{{W^{tot}_I}}\circ({\Omega^{tot}}^* - inv(T_v(\hat{\Sigma}^{tot})))||_{\infty}, ||\dfrac{1}{W^{tot}_S} \circ({\Omega^{tot}}^* - inv(T_v(\hat{\Sigma}^{tot})))||_{\infty})$, then the estimated solution $\hat{\Omega}^{tot}$ satisfies the following error bounds:

{\footnotesize
\begin{equation}
\label{eq:eerate}
\begin{split}
&||\hat{\Omega}^{tot} - {\Omega^{tot}}^*||_{F} \le 4\sqrt{k_i+k_s}\lambda_n \\
&\max(||\dfrac{1}{W^{tot}_I} \circ( \hat{\Omega}^{tot} - {\Omega^{tot}}^*)||_{\infty}, ||\dfrac{1}{W^{tot}_S}\circ (\hat{\Omega}^{tot}-{\Omega^{tot}}^*||_{\infty})\\
&\qquad\qquad\qquad\qquad\qquad\qquad\qquad\qquad\qquad\qquad\qquad
\le 2\lambda_n \\
&||\dfrac{1}{W^{tot}_I} \circ( \hat{\Omega}^{tot}_I - {\Omega^{tot}_I}^*)||_1 + ||\dfrac{1}{W^{tot}_S}\circ (\hat{\Omega}^{tot}_S-{\Omega^{tot}_S}^*)||_1\\
&\qquad\qquad\qquad\qquad\qquad\qquad\qquad\qquad
\le 8(k_i+k_s)\lambda_n 
\end{split}
\end{equation}
}
\vspace{-2mm}
\end{theorem}
\begin{proof}
\vspace{-5mm}
See detailed proof in Section~\ref{seca:proof}
\end{proof}

\rref{theo:jeek} provides a general bound for any selection of $\lambda_n$. The bound of $\lambda_n$ is controlled by the distance between ${\Omega^{tot}}^*$ and $inv(T_v(\hat{\Sigma}^{tot}))$. We then extend \rref{theo:jeek} to derive the statistical convergence rate of JEEK. This gives us the following corollary:
\begin{corollary}
\label{cor:1}
    Suppose the high-dimensional setting, i.e., $p > \max(n_i)$. Let $v:= a\sqrt{\frac{\log (Kp)}{n_{tot}}}$. Then for $\lambda_n := \frac{8\kappa_1 a}{\kappa_2}\sqrt{\frac{\log (Kp)}{n_{tot}}}$ and $n_{tot} > c \log Kp$, with a probability of at least $1-2C_1\exp(-C_2Kp\log (Kp))$, the estimated optimal solution $\hat{\Omega}^{tot}$ has the following error bound:\\
{\small    
 \begin{equation}
\label{eq:eerate}
\begin{split}
||\hat{\Omega}^{tot} -& {\Omega^{tot}}^*||_F 
\\& \le  \frac{16\kappa_1 a\max\limits_{j,k}({W^{tot}_I}_{j,k},{W^{tot}_S}_{j,k})}{\kappa_2}\sqrt{\frac{(k_i+k_s) \log (Kp) }{n_{tot}}}
\end{split}
\end{equation}}
where $a$, $c$, $\kappa_1$ and $\kappa_2$ are constants. 
\end{corollary}
\begin{proof}
See detailed proof in Section~\ref{proof:l2} (especially from~ \eref{eq:proof18} to \eref{eq:bound}).
\end{proof}

\paragraph{Bayesian View of JEEK:}
In \sref{sec:Bayes} we provide a direct Bayesian interpretation of JEEK through the perspective of hierarchical Bayesian modeling. Our hierarchical Bayesian interpretation nicely explains the assumptions we make in JEEK. 


\section{Connecting to Relevant Studies}
\label{sec:rel}

JEEK is closely related to a few state-of-the-art studies summarized in Table~\ref{tab:comcom}. We compare the time complexity and functional properties of JEEK versus these studies.

\begin{table*}[t]
\centering
\caption{Compare JEEK versus baselines.  Here $T$ is the number of iterations.}
\label{tab:comcom}
{\footnotesize 
\begin{tabular}{||p{43mm}|p{25mm}||c|c|c|c|c||}
\hline \hline 
Method& JEEK  & W-SIMULE & JGL & FASJEM & NAK (run $K$ times)\\ \hline \hline 
Time Complexity & $O(K^4p^2)$ {\tiny ($\Rightarrow O(K^4)$ if parallelizing completely)} & $O(K^4p^5)$ & $O(T\times K p^3)$  & $O(T\times Kp^2)$ &  $O(Knp^3+Kp^4)$ \\ \hline
Additional  Knowledge  & YES & YES & NO & NO & YES \\ \hline
\end{tabular}
}
\vspace{-5mm}
\end{table*}

\vspace{1pt}

\paragraph{NAK: \cite{bu2017integrating}}
For the single task sGGM, one recent study~\cite{bu2017integrating} (following ideas from \cite{shimamura2007weighted}) proposed to integrating Additional Knowledge  (NAK)into
estimation of graphical models through a weighted Neighbourhood selection formulation  (NAK) as: $\hat{\beta}^j = \argmin\limits_{\beta,\beta_j = 0} \frac{1}{2}||X^j - X\beta||_2^2 + ||\mathbf{r}_j \circ \beta||_1$. 
 NAK is designed for estimating brain connectivity networks from homogeneous samples and incorporate distance knowledge as weight vectors. \footnote{Here $\hat{\beta}^j$ indicates the sparsity of $j$-th column of a single $\hat{\Omega}$. Namely, $\hat{\beta}^j_k = 0$ if and only if $\hat{\Omega}_{k,j} = 0$. $\mathbf{r}_j$ is a weight vector as the additional knowledge The NAK formulation can be solved by a classic Lasso solver like glmnet.}  In experiments, we compare JEEK to NAK (by running NAK R package  $K$ times) on multiple synthetic datasets of simulated samples about brain regions. The data simulation strategy was   suggested by ~\cite{bu2017integrating}. Same as the NAK~\cite{bu2017integrating}, we use the spatial distance among brain regions as additional knowledge in JEEK. 

\vspace{1pt}

\paragraph{W-SIMULE: ~\cite{singh2017constrained}}
Like JEEK, one recent study ~\cite{singh2017constrained} of multi-sGGMs (following ideas from \cite{wang2017constrained}) also assumed that $\Omega^{(i)} = \Omega^{(i)}_I + \Omega_S$ and incorporated spatial distance knowledge in their convex formulation for joint discovery of heterogeneous neural connectivity graphs. This study, with name W-SIMULE (Weighted model for Shared and Individual parts of MULtiple graphs Explicitly) uses a weighted constrained $\ell_1$ minimization:  
{\scriptsize 
\begin{align}
\label{eq:wsimule}
\argmin\limits_{\Omega^{(i)}_I,\Omega_S}&\sum\limits_i ||W\circ\Omega^{(i)}_I||_1+ \epsilon K||W\circ\Omega_S||_1  \\
\text{Subject to:} \; & ||\Sigma^{(i)}(\Omega^{(i)}_I + \Omega_S) - I||_{\infty} \le \lambda_{n}, \; i =
1,\dots,K \nonumber
\end{align}}
 W-SIMULE simply includes the additional knowledge as a weight matrix $W$. \footnote{It can be  solved by any linear programming solver and can be column-wise paralleled. 
 However, it is very slow when $p > 200$ due to the expensive computation cost $O(K^4p^5)$.}

Different from W-SIMULE, JEEK separates the knowledge of individual context and the shared using different weight matrices. While W-SIMULE also minimizes a weighted $\ell$1 norm, its constraint optimization term is entirely different from JEEK. The formulation difference makes the optimization of JEEK much faster and more scalable than W-SIMULE (\sref{sec:exp}). We have provided a complete theoretical analysis of error bounds of JEEK, while W-SIMULE provided no theoretical results. Empirically, we compare JEEK with W-SIMULE R package from ~\cite{singh2017constrained} in the experiments. 

\vspace{1pt}

\paragraph{JGL: \cite{danaher2013joint}:}~Regularized MLE based multi-sGGMs Studies mostly follow the so called joint graphical lasso (JGL) formulation as ~\eref{eq:ll}:
{\scriptsize
\begin{equation}
\label{eq:ll}
\begin{split}
\argmin\limits_{\Omega^{(i)} \succ 0} &\sum\limits_{i = 1}^{K} (-L(\Omega^{(i)}) + 
\lambda_n \sum\limits_{i = 1}^K ||\Omega^{(i)}||_1 \\
& + \lambda_n' \mathcal{R}'(\Omega^{(1)}, \Omega^{(2)},
\dots,\Omega^{(K)})
\end{split}
\end{equation}}
$\mathcal{R}'(\cdot)$ is the second penalty function for enforcing some structural assumption of group property among the multiple graphs. One caveat of JGL is that 
$\mathcal{R}'(\cdot)$ cannot model explicit additional knowledge. For instance,it can not incorporate the information of a few known hub nodes shared by the contexts. In experiments, we compare JEEK to JGL-co-hub and JGL-perturb-hub toolbox provided by ~\cite{mohan2013node}.

\vspace{1pt}
\paragraph{FASJEM: \cite{wang2017fast}}~One very recent study extended JGL using so-called Elementary superposition-structured moment estimator formulation as~\eref{eq:fasjem}:
{\scriptsize
\begin{equation}
    \label{eq:fasjem}
    \begin{split}
        &\argmin\limits_{\Omega_{tot}} ||\Omega_{tot}||_1 + \epsilon \mathcal{R}'(\Omega_{tot})\\
        & s.t. ||\Omega_{tot} - \text{inv}(T_v(\hat{\Sigma}_{tot}))||_{\infty} \le \lambda_n\\
        & \mathcal{R}'^*(\Omega_{tot} - \text{inv}(T_v(\hat{\Sigma}_{tot}))) \le \epsilon\lambda_n
    \end{split}
\end{equation}
}
FASJEM is much faster and more scalable than the JGL estimators.  However like JGL estimators it can not model additional knowledge and its optimization needs to be carefully re-designed for different $\mathcal{R}'(\cdot)$. \footnote{FASJEM extends JGL into multiple independent group-entry wise optimization just like JEEK. Here $\mathcal{R}^{'*}(\cdot)$ is the dual norm of $\mathcal{R}'(\cdot)$. Because \cite{wang2017fast} only designs the optimization of two cases (group,2 and group,inf), we can not use it as a baseline. }

\vspace{3pt}

Both NAK and W-SIMULE only explored the formulation for estimating neural connectivity graphs using spatial information as additional knowledge. Differently our experiments (\sref{sec:exp}) extend the weight-as-knowledge formulation on weights as distance, as shared hub knowledge, as perturbed hub knowledge, and as nodes' grouping information  (e.g., multiple genes are known to be in the same pathway). This has largely extends the previous studies in showing the real-world adaptivity of the proposed formulation. JEEK elegantly formulates existing knowledge based on the problem at hand and avoid the need to design knowledge-specific optimization.

\section{Experiments}
\label{sec:exp}
We empirically evaluate JEEK and baselines on four types of datasets, including two groups of synthetic data, one real-world fMRI dataset for brain connectivity estimation  and one real-world genomics
dataset for estimating interaction among regulatory genes (results in \sref{subsec:exp4}). 
In order to incorporating various types of knowledge, we provide five different designs of the weight matrices in Section~\ref{sec:designW}. Details of experimental setup, metrics and hyper-parameter tuning are included in \sref{sec:expsetMore}. Baselines used in our experiments have been explained in details by \sref{sec:rel}. We also use JEEK with no additional knowledge (JEEK-NK) as a baseline.

JEEK is available as the R package 'jeek' in CRAN. 

\begin{figure*}[ht]
    \centering
    \includegraphics[width=0.99\textwidth]{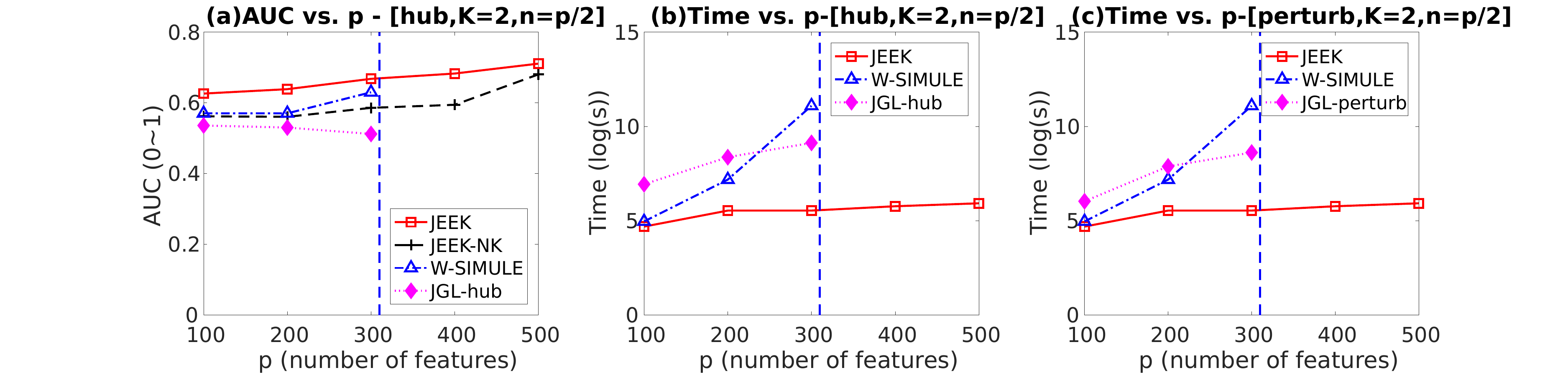}
    \caption{Performance comparison on simulation Datasets using co-Hub Knowledge: AUC vs. 
    Time when varying number of nodes $p$.}
    \label{fig:hubSumP}
\vspace{-5mm}
\end{figure*}

\subsection{Experiment: Simulated Samples with Known Hubs as Knowledge}
\label{subsec:exp3}

Inspired the JGL-co-hub and JGL-perturb-hub toolbox (JGL-node) provided by ~\cite{mohan2013node}, we empirically show JEEK's ability to model known co-hub or perturbed-hub nodes as knowledge when estimating multiple sGGMs. We generate multiple simulated Gaussian datasets through the random graph model \cite{rothman2008sparse} to simulate both the co-hub and perturbed-hub graph structures (details in \ref{subsec:exp3more}). We use JGL-node package, W-SIMULE and JEEK-NK as baselines for this set of experiments. The weights in $\{W_I^{tot}, W_S^{tot}\}$ are designed using the strategy proposed in~\sref{sec:designW}.

We use AUC score (to reflect the consistency and variance of a method's performance when varying its important hyper-parameter) and computational time cost to compare JEEK with baselines. We compare all methods on many simulated cases by varying $p$ from  the set $\{100,200,300,400,500\}$ and the number of tasks $K$ from the set $\{2,3,4\}$. In Figure~\ref{fig:hubSumP} and Figure~\ref{fig:sim_hubPPI}(a)(b), JEEK consistently achieves higher AUC-scores than the baselines JGL, JEEK-NK and W-SIMULE for all cases.  JEEK is more than $~10$ times faster than the baselines on average. In Figure~\ref{fig:hubSumP}, for each $p>300$ case (with $n=p/2$), W-SIMULE takes more than one month and JGL takes more than one day. Therefore we can not show them with $p>300$.

\subsection{Experiment: Gene Interaction Network from Real-World Genomics Data}
\label{subsec:exp4}

Next, we apply JEEK and the baselines on one real-world  biomedical data about gene expression profiles across two different cell types. We explored two different types of knowledge:  (1) Known edges and  (2) Known group about genes. Figure~\ref{fig:sim_hubPPI}(c) shows that JEEK has lower time cost and recovers more interactions than baselines 
(higher number of matched edges to the existing bio-databases.). More results are in Appendix~\sref{subsec:exp4more} and the design of weight matrices for this case is in \sref{sec:designW}.

\subsection{Experiment: Simulated Data about Brain Connectivity with Distance as Knowledge}
\label{subsec:exp2}

Following \cite{bu2017integrating}, we use one known Euclidean distance between human brain regions as additional knowledge $W$ and use it to generate multiple simulated datasets (details in Section~\ref{subsec:exp2more}). We compare JEEK with the baselines regarding (a) Scalability (computational time cost), and (b) effectiveness (F1-score, because NAK package does not allow AUC calculation).  For each simulation case, the computation time for each estimator is the summation of a method's execution time over all values of $\lambda_n$. 
Figure~\ref{fig:sim_brain}(a)(b) show clearly that JEEK outperforms its baselines. JEEK has a consistently higher F1-Score and is almost $6$ times faster than W-SIMULE in the high dimensional case. JEEK performs better than JEEK-NK, confirming the advantage of integrating additional distance knowledge. While NAK is fast, its F1-Score is nearly $0$ and hence, not useful for multi-sGGM structure learning.

\subsection{Experiment: Functional Connectivity Estimation from Real-World Brain fMRI Data}
\label{subsec:exp1}

We evaluate JEEK and relevant baselines for a classification task on one real-world publicly available resting-state fMRI dataset: ABIDE\cite{di2014autism}. The ABIDE data aims to understand human brain connectivity and how it reflects neural disorders \cite{van2013wu}. ABIDE includes two groups of human subjects: autism and control, and therefore we formulate it as $K=2$ graph estimation. We utilize the spatial distance between human brain regions as additional knowledge for estimating functional connectivity edges among brain regions. We use Linear Discriminant Analysis (LDA) for a downstream classification task aiming to assess the ability of a graph estimator to learn the differential patterns of the connectome structures.  (Details of the ABIDE dataset, baselines, design of the additional knowledge $W$ matrix, cross-validation and LDA classification method are in \sref{subsec:exp1more}.)

Figure~\ref{fig:sim_brain}(c) compares JEEK and three baselines: JEEK-NK, W-SIMULE and W-SIMULE with no additional knowledge (W-SIMULE-NK). JEEK yields a classification accuracy of 58.62\% for distinguishing the autism subjects versus the control subjects,  clearly outperforming JEEK-NK 
and W-SIMULE-NK. 
JEEK is roughly $7$ times faster than the W-SIMULE estimators, locating at the top left region in Figure~\ref{fig:sim_brain}(c) (higher classification accuracy and lower time cost).  We also experimented with variations of the $W$ matrix and found the classification results are fairly robust to the variations of $W$ (\sref{subsec:exp1more}).

\section{Conclusions}
\label{sec:concl}
\vspace{-3pt}

We propose a novel method, JEEK, to incorporate additional knowledge in estimating multi-sGGMs. JEEK achieves the same asymptotic convergence rate as the state-of-the-art. 
Our experiments has showcased using weights for describing pairwise knowledge among brain regions, for shared hub knowledge, for perturbed hub knowledge, for describing group information among nodes (e.g., genes known to be in the same pathway), and for using  known interaction edges as the knowledge.


\begin{figure*}[htb]
    \centering
    \includegraphics[width=0.32\textwidth]{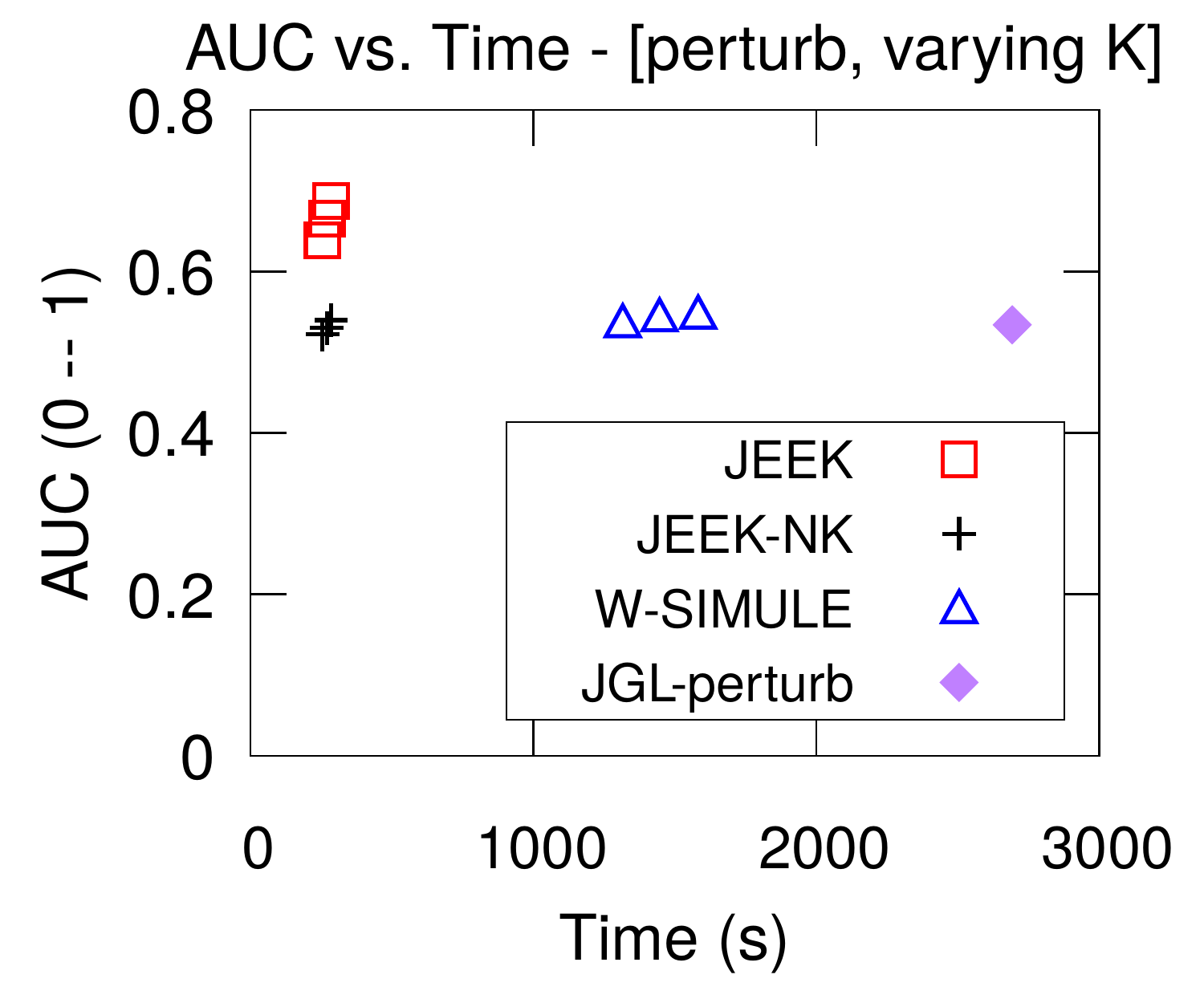}
    \includegraphics[width=0.32\textwidth]{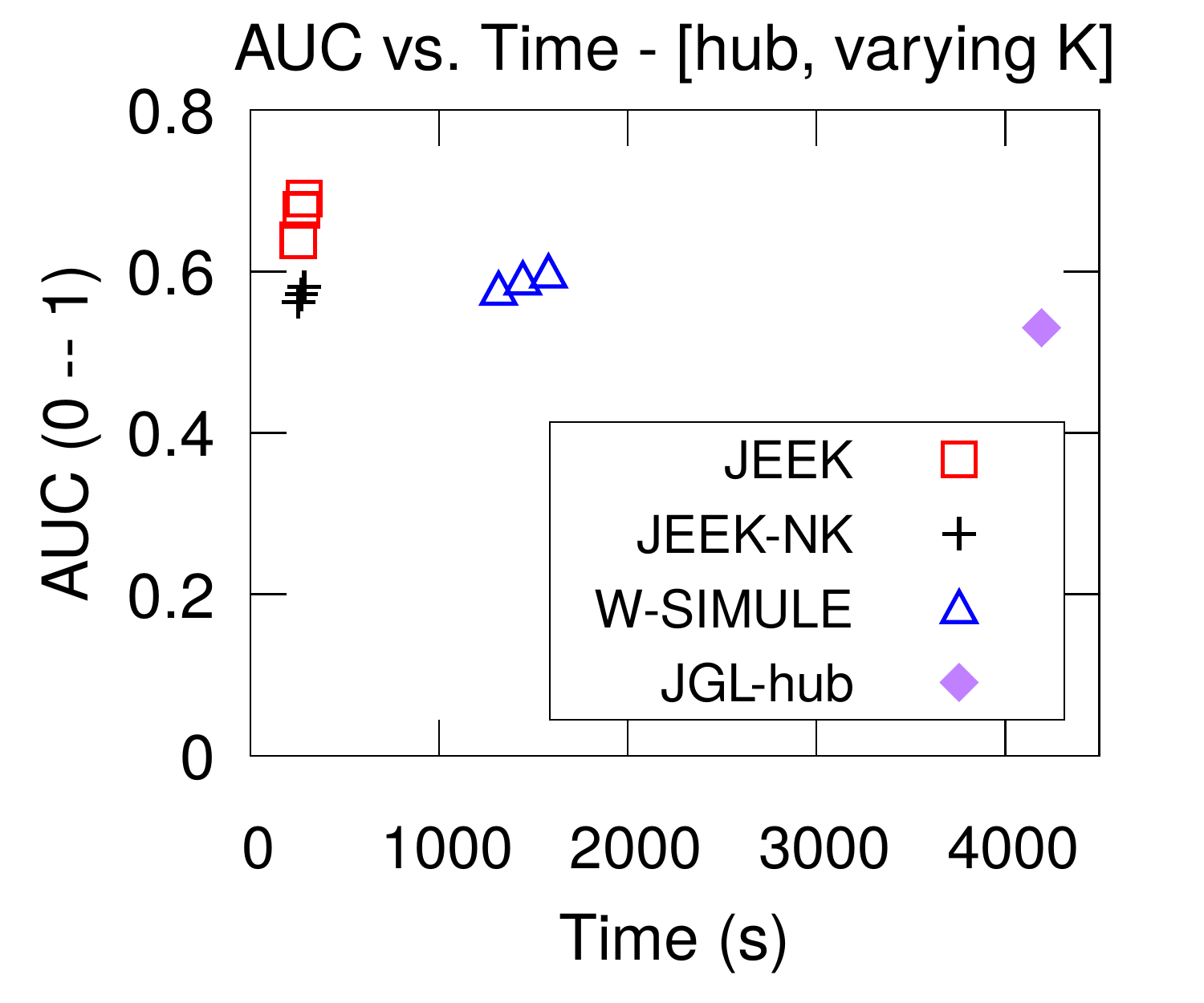}
    \includegraphics[width=0.35\textwidth]{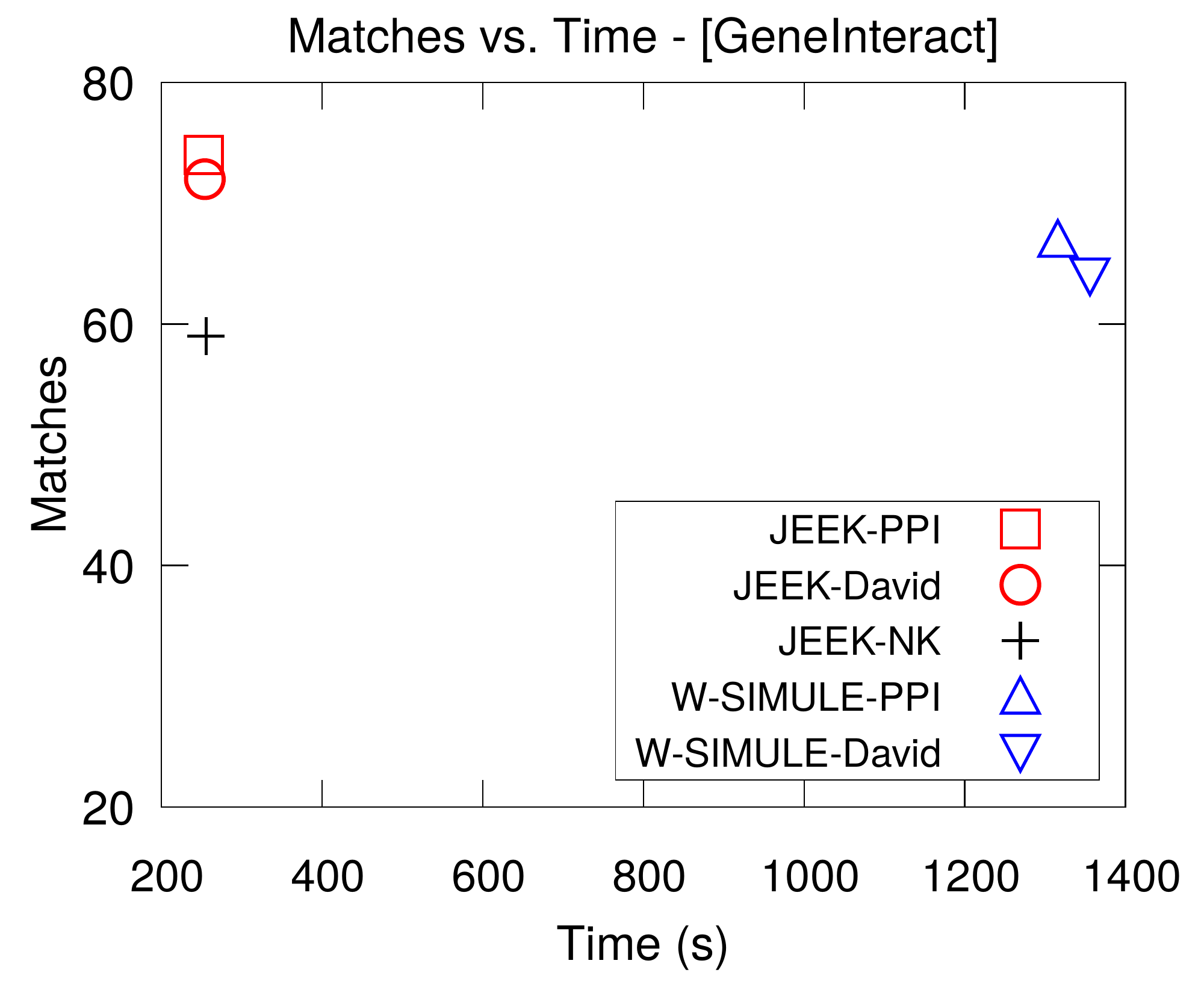}\\
    (a) \hspace{5cm} (b) \hspace{5cm} (c)\\   
    \vspace{-3mm}
    \caption{(a)(b) Performance comparison on simulation Datasets about hubs: AUC vs. Time when varying number of tasks $K$. (a) is the perturbed hub cases and (b) is for the co-hub cases. (c) Performance comparison on one real-world gene expression dataset with two cell types. Two type knowledge are used to cover one fifth of the nodes, therefore each method corresponds to two performance points.}
    \label{fig:sim_hubPPI}
\end{figure*}

\begin{figure*}[htb]
    \centering
    \includegraphics[width=0.33\textwidth]{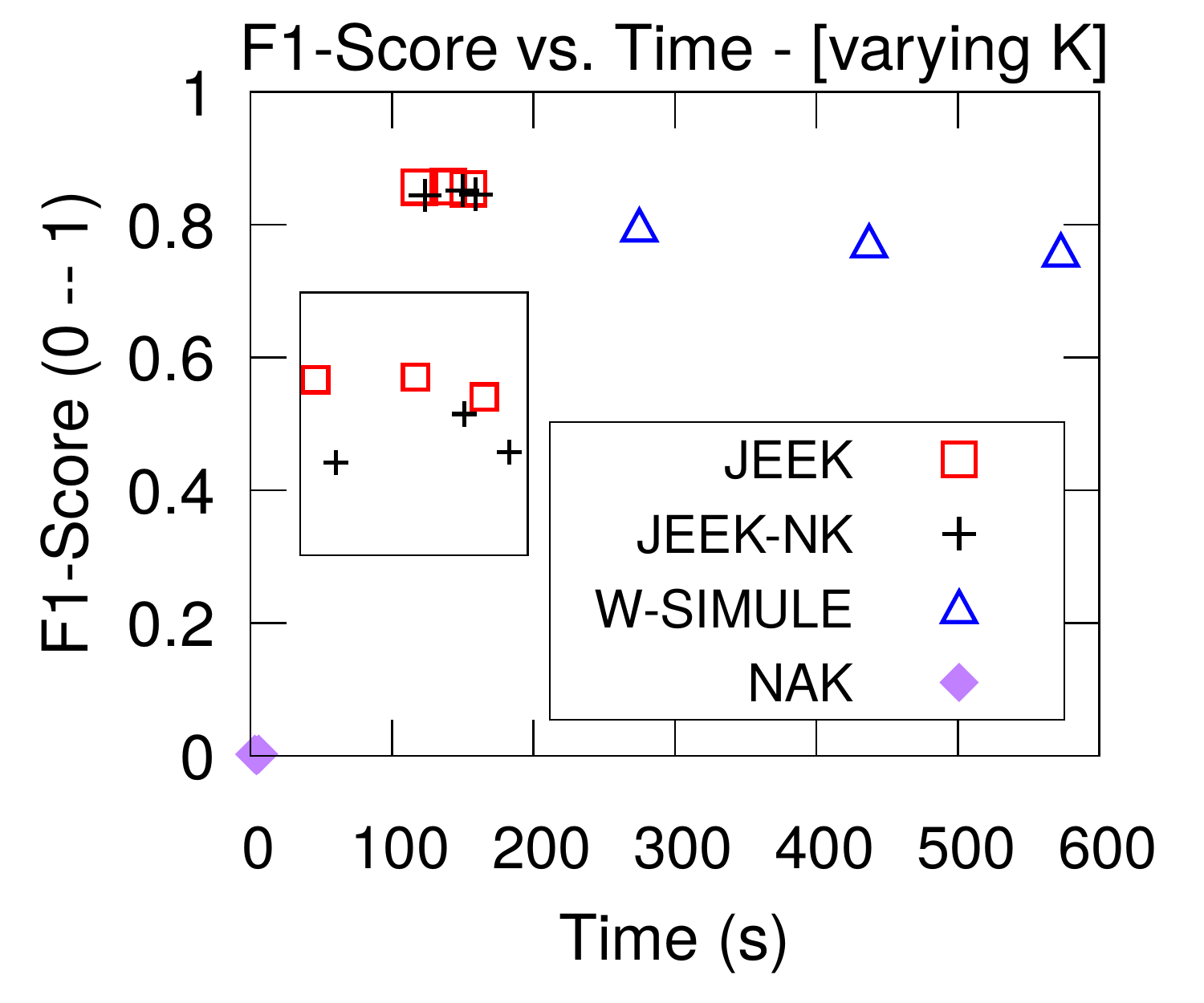}
    \includegraphics[width=0.33\textwidth]{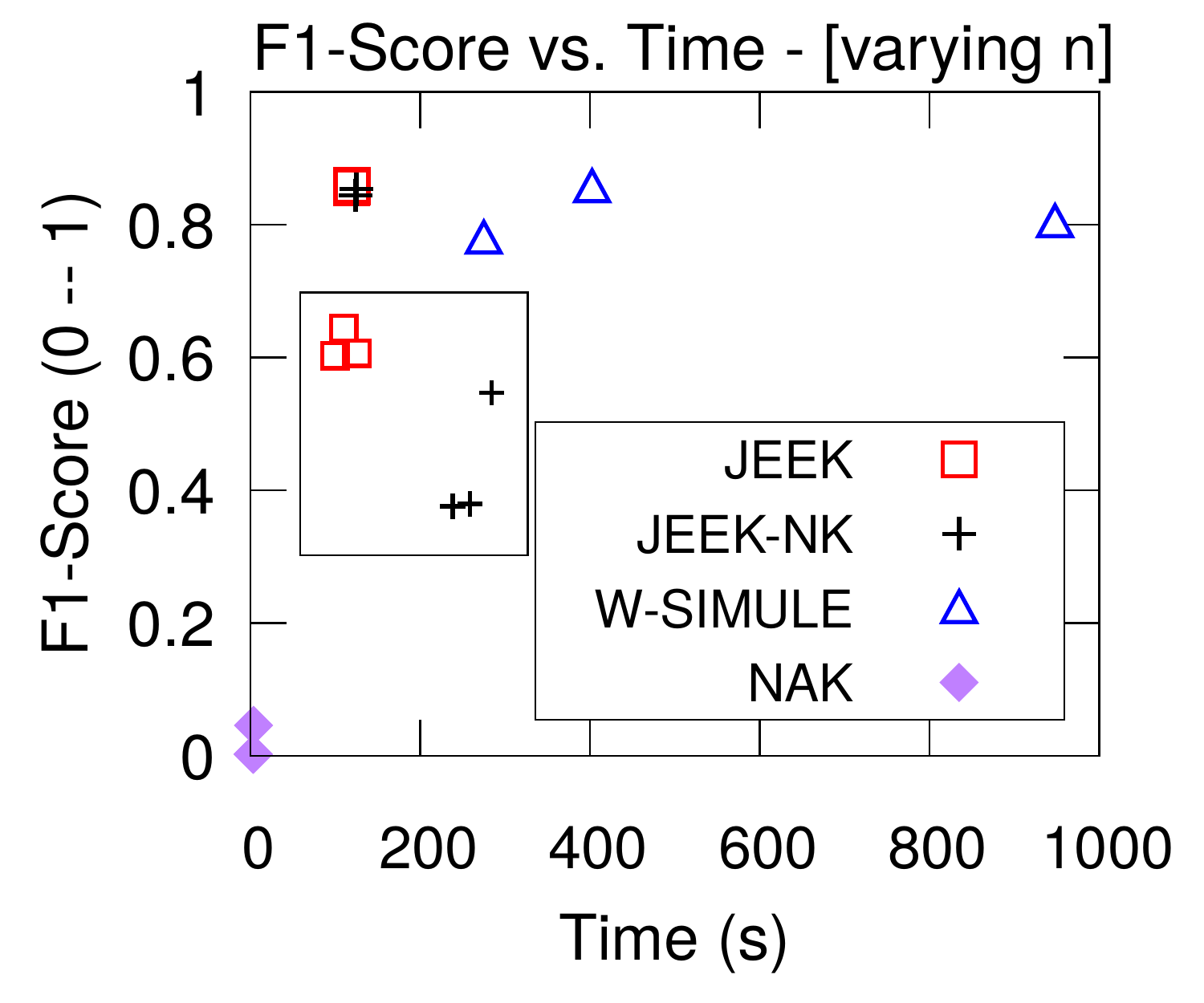}
    \includegraphics[width=0.33\textwidth]{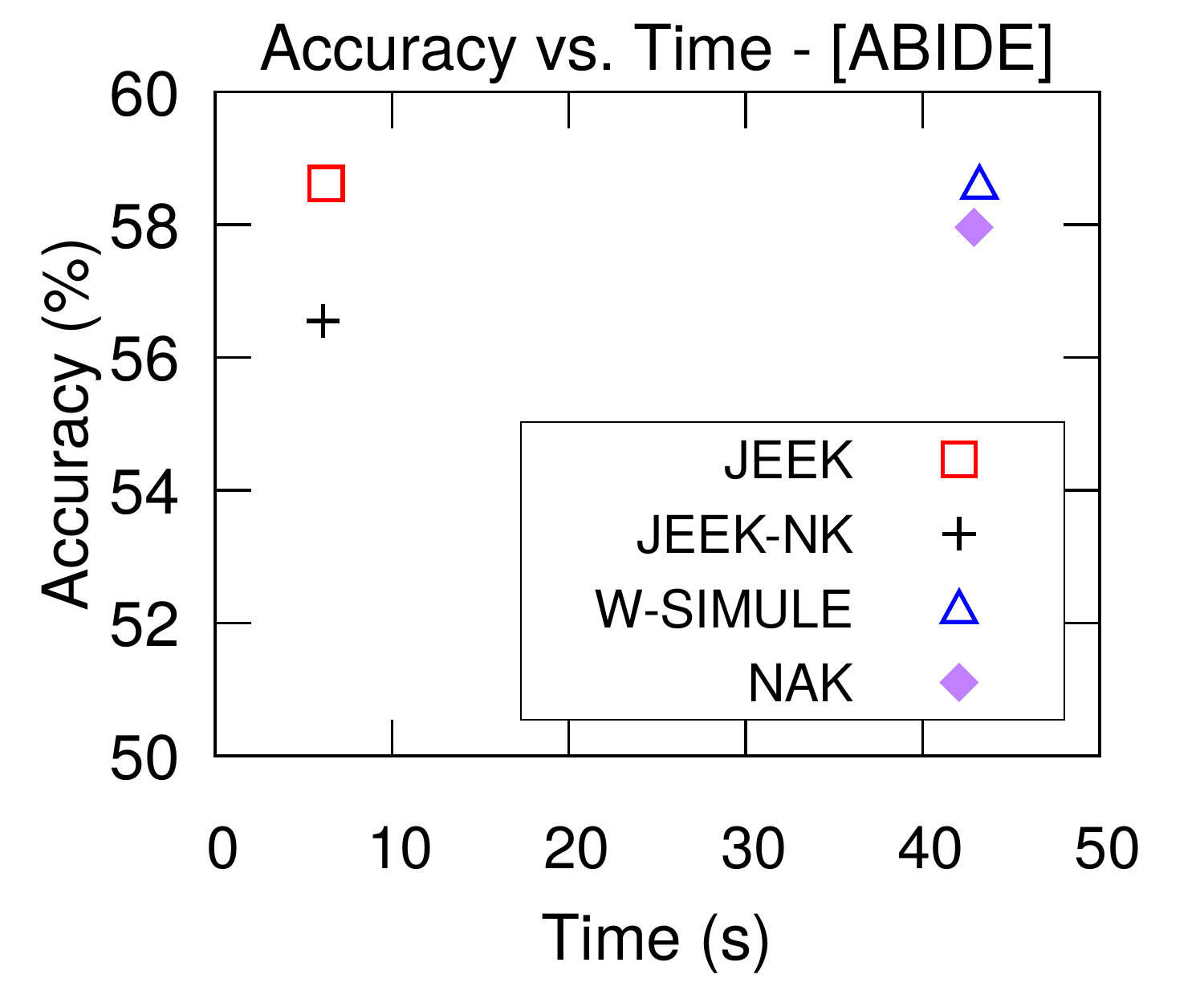}\\
    (a) \hspace{5cm} (b) \hspace{5cm} (c)\\
     \vspace{-3mm}
   \caption{Experimental Results on Simulated Brain Datasets and on ABIDE. (a) Performance obtained on  simulated brain samples with respect to F1-score vs. computational time cost when varying the number of tasks $K$.  (b) Performance obtained on  simulated brain samples with respect to F1-score vs. computational time cost when varying the number of samples $n$. In both (a) and (b) the smaller box shows an enlarged view comparing JEEK and JEEK-NK points. All JEEK points are in the top left region indicating higher F1-score and lower computational cost. (c). On ABIDE, JEEK outperforms the baseline methods in both classification accuracy and running time cost. JEEK and JEEK-NK points in the top left region and JEEK points are higher in terms of $y$-axis positions.}
    \label{fig:sim_brain}
\end{figure*}

\vspace{-1mm}


{\noindent \textbf{Acknowledgement:} \textit{This work was supported by the National Science Foundation under NSF CAREER award No. 1453580. Any Opinions, findings and conclusions or recommendations expressed in this material are those of the author(s) and do not necessarily reflect those of the National Science Foundation.}}


\textbf{\Large Appendix: }

\section{More about Method}
\label{sec:moremeth}

\textit{Notations: }  $X^{(i)}_{n_i \times p}$ is the data matrix
for the $i$-th task, which includes $n_i$
 data samples being described by $p$ different feature variables. Then $n_{tot} =\sum\limits_{i=1}^{K}n_i$ is the total number of data samples. We use notation $\Omega^{(i)}$ for the precision matrices and $\widehat{\Sigma}^{(i)}$ for the estimated covariance matrices. 
Given a $p$-dimensional vector $\x = (x_1, x_2, \dots, x_p)^T \in
\RR^p$, we denote the $l_1$-norm of
$x$ as $||\x||_1 = \sum\limits_i|x_i|$. $||\x||_{\infty} = \max\limits_i|x_i|$ is the $l_{\infty}$-norm
of $\x$. Similarly, for a matrix $X$, let $||X||_1 = \sum\limits_{i,j}|X_{i,j}|$ be the $\ell_1$-norm of $X$ and $||X||_{\infty} = \max\limits_{i,j}|X_{i,j}|$ be the $\ell_{\infty}$-norm of $X$. $||X||_F = \sqrt{\sum\limits_i \sum\limits_j X^2_{i,j}}$

\subsection{More about Solving JEEK}

In ~\eref{eq:JEEK-par}, let $a_i = a_i^+ - a_i^-$ and $b = b^+ - b^-$. If $a_i \ge 0$, then $a_i^+ = a_i$ and $a_i^- = 0$. If $a_i < 0$, then $a_i^+ = 0$ and $a_i^- = -a_i$. The $b^+$ and $b^-$ have the similar definition. Then~\eref{eq:JEEK-par} can be solved by the following small linear programming problem.
\begin{equation}
\label{eq:JEEK-lp}
    \begin{split}
    & \argmin\limits_{a_i,b}\sum\limits_i (w_i a_i^+ + w_i a_i^-) + Kw_s b^+ + Kw_sb^-  \\
    \text{Subject to:} \; & a_i^+ - a_i^- + b^+ - b^- \le c_i + \lambda_{n}\min(w_i,w_s), \\
    & a_i^+ - a_i^- + b^+ - b^- \ge c_i - \lambda_{n}\min(w_i,w_s), \\
    & a_i^+, a_i^-, b^+, b^- \ge 0
    \\& i =
    1,\dots,K \nonumber
    \end{split}
\end{equation}

\subsection{JEEK is Group entry-wise and parallelizing optimizable}
\label{sec:parallel}
JEEK can be easily paralleled. Essentially we just need to revise the ``For loop'' of step 6 and step 7 in Algorithm~\ref{alg:JEEK} into, for instance, ``entry per machine"  ``entry per core''.  Now We prove that JEEK is group entry-wise and parallelizing optimizable. We prove that our estimator can be optimized asynchronously in a group entry-wise manner. 
\begin{theorem}
(\textbf{JEEK is Group entry-wise optimizable})
 Suppose we use JEEK to infer multiple inverse of covariance matrices summarized as $\hat{\Omega}_{tot}$. $\{ [\hat{\Omega}_{I}^{(i)}]_{j,k}, [\hat{\Omega}_S]_{j,k} | i = 1,\dots,K \}$. describes a group of $K+1$ entries at $(j,k)$ position. Varying $j \in \{ 1,2,\dots,p \}$ and $k \in \{ 1,2,\dots,p \}$, we have a total of $p \times p$ groups. If these groups are independently estimated by JEEK, then we have,
\begin{equation}
 \bigcup\limits_{j = 1}^{p}\bigcup\limits_{k = 1}^{p} \{ ([\hat{\Omega}^{(i)}_I]_{j,k} + [\hat{\Omega}_S]_{j, k})|i = 1,\dots,K \} = \hat{\Omega}_{tot}. 
 \end{equation}
\label{entry-wise}
\end{theorem}
\begin{proof}
\eref{eq:JEEK-par} are the small sub-linear programming problems on each group of entries. 
\end{proof}

\subsection{Extending JEEK with Structured Norms}

We can add more flexibility into the JEEK by adding structured norms like those second normalization functions used in JGL. This will extend JEEK to the following formulation:
\begin{equation}
\begin{split}
    \label{eq:JEEK-ex}
     \argmin\limits_{\Omega^{tot}_I, \Omega^{tot}_S} & ||W^{tot}_I \circ \Omega^{tot}_I||_1 + ||W^{tot}_S\circ \Omega^{tot}_S|| + \epsilon \mathcal{R}'(\Omega^{tot}) \\
    \text{Subject to: } & ||\dfrac{1}{W^{tot}_I} \circ (\Omega^{tot} - inv(T_v(\hat{\Sigma}^{tot}))) ||_{\infty} \le \lambda_n\\
    & ||\dfrac{1}{W^{tot}_S} \circ (\Omega^{tot} - inv(T_v(\hat{\Sigma}^{tot}))) ||_{\infty} \le \lambda_n\\
    & {\mathcal{R}^*}'(\Omega^{tot}) \le \epsilon \lambda_n
\end{split}
\end{equation}
Here, $\mathcal{R}'$ needs to consider $\Omega^{tot}$.  We propose two ways to solve~\eref{eq:JEEK-ex}. 
(1) The first is to use the parallelized proximal algorithm directly. However, this requires the kw-norm has a closed-form proximity, which has not been discovered. 
(2) In the second strategy we assume each weighted $\ell_1$ norm (either $\Omega^{(i)_I}$ or $\Omega_S$) in the objective of ~\eref{eq:JEEK-ex}  as an indepedent regularizer. However, this increases the number of proximities we need to calculate per iteration to $K+1$. Both  two solutions make the extend-JEEK algorithm less fast or less scalable. Therefore, We choose not to introduce this work in this paper.

\section{Connecting to the Bayesian statistics}
\label{sec:Bayes}

Our approach has a close connection to a  hierarchical Bayesian model perspective. We show that the additional knowledge weight matrices are also the parameters of the prior distribution of $\Omega_I^{(i)},{\Omega_S}$.
In our formulation\eref{eq:JEEK}, $W_I^{(i)}, W_S$ are the additional knowledge weight matrices. From a hierarchical Bayesian view, the first level of the prior is a Gaussian distribution and the second level is a Laplace distribution. In the following section, we show that $W_I^{(i)}, W_S$ are also the parameters of Laplace distributions, which is a prior distribution of $\Omega_I^{(i)},{\Omega_S}$.

Since by the definition, ${\Omega_I^{(i)}}_{j,k}{\Omega_S}_{j,k} = 0$. There are only two possible situations:

Case I (${\Omega_I^{(i)}}_{j,k} = 0$):
\begin{equation}
\label{eq:bayes1}
    X^{(i)} | \mu^{(i)}, \Omega^{(i)} \sim N(\mu^{(i)}, (\Omega^{(i)})^{-1})
\end{equation}

\begin{equation}
\begin{split}
    \label{eq:bayes2}
    \Omega^{(i)}_{j,k} | \mu^{(i)}, {W_I^{(i)}}_{j,k}, {W_S}_{j,k} = {\Omega_S}_{j,k} | \mu^{(i)}, {W_S}_{j,k}
\end{split}
\end{equation}
\begin{equation}
\begin{split}
    \label{eq:bayes3}
    &p({\Omega_S}_{j,k}|\mu^{(i)}, {W_S}_{j,k}) \\ 
    &\propto e^{-({W_S}_{j,k} |{\Omega_S}_{j,k}|)}
\end{split}
\end{equation}
Here ${\Omega_S}_{j,k}|\mu^{(i)}, {W_S}_{j,k}$ follows a Laplace distribution with mean $0$. $1/{W_S}_{j,k} > 0$ is the diversity parameter. The larger ${W_S}_{j,k}$ is, the distribution of ${\Omega_S}_{j,k}|\mu^{(i)}, {W_S}_{j,k}$ more likely concentrate on the $0$. Namely, there will be the higher density for ${\Omega_S}_{j,k} = 0|\mu^{(i)}, {W_S}_{j,k}$.

Case II (${\Omega_S}_{j,k} = 0$):
\begin{equation}
\label{eq:bayes4}
    X^{(i)} | \mu^{(i)}, \Omega^{(i)} \sim N(\mu^{(i)}, (\Omega^{(i)})^{-1})
\end{equation}

\begin{equation}
\begin{split}
    \label{eq:bayes5}
    \Omega^{(i)}_{j,k} | \mu^{(i)}, {W_I^{(i)}}_{j,k}, {W_S}_{j,k} = {\Omega_I^{(i)}}_{j,k} | \mu^{(i)}, {W_I^{(i)}}_{j,k}
\end{split}
\end{equation}

\begin{equation}
\begin{split}
    \label{eq:bayes6}
    &p({\Omega_I^{(i)}}_{j,k}|\mu^{(i)}, {W_I^{(i)}}_{j,k}) \\ 
    &\propto e^{-({W_I^{(i)}}_{j,k}|{\Omega_I^{(i)}}_{j,k}|)}
\end{split}
\end{equation}
Here ${\Omega_I^{(i)}}_{j,k}|\mu^{(i)}, {W_I^{(i)}}_{j,k}$ follows a Laplace distribution with mean $0$. $1/{W_I^{(i)}}_{j,k} > 0$ is the diversity parameter. The larger ${W_I^{(i)}}_{j,k}$ is, the distribution of ${\Omega_I^{(i)}}_{j,k}|\mu^{(i)}, {W_I^{(i)}}_{j,k}$ more likely concentrate on the $0$. Namely, there will be the higher density for ${\Omega_I^{(i)}}_{j,k} = 0|\mu^{(i)}, {W_I^{(i)}}_{j,k}$.

Therefore, we can combine the above two cases into the following one equation.

\begin{equation}
\begin{split}
    \label{eq:bayes7}
    &p(\Omega^{(i)}_{j,k}|\mu^{(i)}, {W_I^{(i)}}_{j,k}, {W_S}_{j,k}) \\ 
    &\propto e^{-({W_I^{(i)}}_{j,k}|{\Omega_I^{(i)}}_{j,k}| + {W_S}_{j,k} |{\Omega_S}_{j,k}|)}
\end{split}
\end{equation}

Our final hierarchical Bayesian formulation consists of the~\eref{eq:bayes1} and~\eref{eq:bayes7}. 
This model is a generalization of the model considered in the seminal paper
on the Bayesian lasso\cite{park2008bayesian}. The parameters ${W_I^{(i)}}_{j,k}, {W_S}_{j,k}$ in our general
model are hyper-parameters that specify the shape of the prior distribution of each edges in $\Omega^{(i)}$. The negative log-posterior distribution of $\Omega^{(i)}$ is now given by:
\begin{equation}
\begin{split}
    \label{eq:bayes8}
    &-\log(\P(\Omega^{(i)}| X^{(i)}, \mu^{(i)}, {W_I^{(i)}}_{j,k}, {W_S}_{j,k})) \\
    &\propto -\log(det({\Omega^{(i)}}^{-1})) + <\Omega^{(i)}, \hat{\Sigma}^{(i)}>\\ 
    &+ \sum\limits_{j,k} ({W_I^{(i)}}_{j,k}|{\Omega_I^{(i)}}_{j,k}| + W_S |{\Omega_S}_{j,k}|)
\end{split}
\end{equation}
~\eref{eq:bayes8} follows a weighted variation  of~\eref{eq:ggm}.

\section{More about Theoretical Analysis}
\label{sec:proof}

\subsection{Theorems and Proofs of three properties of kw-norm}
\label{sec:proofkw}

In this sub-section, we prove the three properties of kw-norm used in Section~\ref{sec:wkn}. 
We then provide the convergence rate of our estimator based on these three properties. 

\begin{itemize}
    \item (i) kw-norm is a norm function if and only if any entries in $W^{tot}_I$ and $W^{tot}_S$ do not equal to $0$.
    \item (ii) If the condition in (i) holds, kw-norm is a decomposable norm.
    \item (iii) If the condition in (i) holds, the dual norm of kw-norm is $\mathcal{R}^*(u) = \max(||W^{tot}_I \circ     u||_{\infty}, ||W^{tot}_S \circ u||_{\infty})$.
\end{itemize}

\subsubsection{Norm:}~ First we prove the correctness of the argument that kw-norm is a norm function by the following theorem:
\begin{theorem}
\label{theo:norm}
~\eref{eq:kw-norm} is a norm if and only if $\forall 1\ge j,k\le p, {W^{(i)}_I}_{jk}\ne 0$, and ${W_{S}}_{j,k}\ne 0$.
\end{theorem}
This theorem gives the sufficient and necessary conditions to make kw-norm (~\eref{eq:kw-norm}) a norm function.

\subsubsection{Decomposable Norm:}~ Then we show that kw-norm is a decomposable norm within a certain subspace. 
Before providing the theorem, we give the structural assumption of the parameter.

\textbf{(IS-Sparsity):} The 'true' parameter for  ${\Omega^{tot}}^*$  ( multiple GGM structures) can be decomposed into two clear structures--${\Omega^{tot}_I}^*$ and ${\Omega^{tot}_S}^*$. ${\Omega^{tot}_I}^*$ is exactly sparse with $k_i$ non-zero entries indexed by a support set $S_I$ and ${\Omega^{tot}_S}^*$ is exactly sparse with $k_s$ non-zero entries indexed by a support set $S_S$. $S_I\bigcap S_S = \emptyset$. All other elements  equal to $0$ (in $(S_I\bigcup S_S)^c$). 
\begin{definition}(IS-subspace)
\label{def:m}
\begin{equation}
   \mathcal{M}(S_I\bigcup S_S) = \{ \theta_{j} = 0| \forall j \notin S_I\bigcup S_S\} 
\end{equation}
\end{definition}

\begin{theorem}
\label{theo:decomp}
~\eref{eq:kw-norm} is a decomposable norm with respect to $\mathcal{M}$ and $\bar{\mathcal{M}}^{\perp}$
\end{theorem}

\subsubsection{Dual Norm of kw-Norm:}~To obtain the final formulation~\eref{eq:JEEK} and 
its statistical convergence rate, we need to derive the dual norm formulation of kw-norm.
\begin{theorem}
\label{theo:dual}
The dual norm of kw-norm (~\eref{eq:kw-norm}) is
\begin{equation}
\label{eq:dual-norm}
\mathcal{R}^*(u) = \max(||\dfrac{1}{W^{tot}_I} \circ     u||_{\infty}, ||\dfrac{1}{W^{tot}_S} \circ u||_{\infty})
\end{equation}
\end{theorem}
The details of the proof are as follows.

\subsubsection{Proof of~\rref{theo:norm}}

\begin{lemma}
\label{le:1}
For kw-norm, ${W^{tot}_I}_{j,k}\ne 0$ and ${W^{tot}_S}_{j,k} \ne 0$ equals to ${W^{tot}_I}_{j,k}> 0$ and ${W^{tot}_S}_{j,k} > 0$.
\end{lemma}
\begin{proof}
If ${W^{tot}_I}_{j,k} < 0$, then $|{W^{tot}_I}_{j,k} {\Omega^{tot}_I}_{j,k}| = |{W^{tot}_I}_{j,k}|| {\Omega^{tot}_I}_{j,k}| = |-{W^{tot}_I}_{j,k} {\Omega^{tot}_I}_{j,k}|$. Notice that $-{W^{tot}_I}_{j,k} > 0$.
\end{proof}

\begin{proof}
To prove the kw-norm is a norm, by~\lref{le:1} the only thing we need to prove is that $f(x) = ||W\circ x||_1$ is a norm function if $W_{i,j} > 0$.
1. $f(ax)= ||aW \circ x||_1 = |a|||W\circ x||_1 = |a|f(x)$.
2. $f(x +y) =||W\circ(x+y)||_1 = ||W\circ x +W\circ y||_1\le ||W\circ x||_1 + ||W\circ y||_1 = f(x) + f(y)$.
3. $f(x) \ge 0$
4. If $f(x) = 0$, then $\sum |W_{i,j}x_{i,j}| = 0$. Since $W_{i,j}\neq 0$, $x_{i,j} =0$. Therefore, $x = 0$.
Based on the above, $f(x)$ is a norm function.
Since summation of norm is still a norm function, kw-norm is a norm function.
\end{proof}

Furthermore, we have the following Lemma:
\begin{lemma}
\label{le:dual}
The dual norm of $f(x) = ||W\circ x||_1$ is 
$$||\dfrac{1}{W}\circ x||_{\infty}$$.
\end{lemma}
\begin{proof}
\begin{align}
f^*(u) & = \sup\limits_{||W\circ x||_1\le 1}<u,x> \\
& \le \sup\limits_{||W\circ x||_1\le 1} (\sum_{k=1,...,p}|w_k x_k|) \max_{k=1,...,p}|\dfrac{1}{w_k} u_k| \\
& = ||\dfrac{1}{W} \circ u||_{\infty} \end{align}
\end{proof}


\subsubsection{Proof of~\rref{theo:decomp}}
\begin{proof}
Assume $u\in \mathcal{M}$ and $v \in \bar{\mathcal{M}}^{\perp}$, $\mathcal{R}(u+v)=||W_I^{tot}\circ (u_I + v_I)||_1+||W_S^{tot}\circ (u_S + v_S)||_1 = ||W_I^{tot}\circ u_I||_1+||W_S^{tot}\circ u_S||_1 + ||W_I^{tot}\circ v_I||_1+||W_S^{tot}\circ v_S||_1 = \mathcal{R}(u)+\mathcal{R}(v)$.
Therefore, kw-norm is a decomposable norm with respect to the subspace pair $(\mathcal{M},\bar{\mathcal{M}}^{\perp})$.
\end{proof}

\subsubsection{Proof of~\rref{theo:dual}}
\begin{proof}
Suppose $\mathcal{R}(\theta) = \sum\limits_{\alpha \in I} c_{\alpha}\mathcal{R}_{\alpha}(\theta_{\alpha})$, where $\sum\limits_{\alpha \in I} \theta_{\alpha} = \theta$. Then the dual norm $\mathcal{R}^{*}(\cdot)$ can be derived by the following equation.
\begin{equation}
    \begin{split}
        \mathcal{R}^*(u) &= \sup\limits_{\theta}\frac{<\theta,u>}{\mathcal{\theta}}\\
        &= \sup\limits_{\theta_{\alpha}} \frac{\sum\limits_{\alpha}<u,\theta_{\alpha}>}{\sum\limits_{\alpha}c_{\alpha}\mathcal{R}_{\alpha}(\theta_{\alpha})}\\
        &= \sup\limits_{\theta_{\alpha}} \frac{\sum\limits_{\alpha}<u/c_{\alpha},\theta_{\alpha}>}{\sum\limits_{\alpha}\mathcal{R}_{\alpha}(\theta_{\alpha})}\\
        &\le \sup\limits_{\theta_{\alpha}} \frac{\sum\limits_{\alpha}\mathcal{R}_{\alpha}^*(u/c_{\alpha})\mathcal{R}(\theta_{\alpha})}{\sum\limits_{\alpha}\mathcal{R}_{\alpha}(\theta_{\alpha})}\\
        &\le \max\limits_{\alpha \in I}\mathcal{R}_{\alpha}^*(u)/c_{\alpha}.
    \end{split}
\end{equation}
Therefore by~\lref{le:dual}, the dual norm of kw-norm is $\mathcal{R}^*(u) = \max(||W^{tot}_I \circ     u||_{\infty}, ||W^{tot}_S \circ u||_{\infty})$.
\end{proof}

\subsection{Appendix: Proofs of Theorems about All Error Bounds of JEEK}
\label{seca:proof}

\subsubsection{Derivation of~\rref{theo:jeek}}
\label{proof:L1}
JEEK formulation~\eref{eq:JEEK} and EE-sGGM \eref{eq:eeggm} are special cases of the following generic formulation: 
\begin{equation}
\label{eq:ee}
  \begin{split}
    &\argmin\limits_{\theta} \mathcal{R}(\theta)\\
    &\text{subject to:} \mathcal{R}^*(\theta -\hat{\theta}_n) \le \lambda_n 
    \end{split}
\end{equation}
Where $\mathcal{R}^*(\cdot)$ is the dual norm of $\mathcal{R}(\cdot)$,  
\begin{equation}
\mathcal{R}^*(v) := \sup\limits_{u \ne 0}\frac{<u,v>}{\mathcal{R}(u)} = \sup\limits_{\mathcal{R}(u) \le 1}<u,v>.
\end{equation}

Connecting \eref{eq:JEEK} and \eref{eq:ee}, $\mathcal{R}()$ is the kw-norm. $\hat{\theta}_n$ represents a close approximation of $\theta^*$.

Following the unified framework \cite{negahban2009unified}, we first decompose the parameter space into a subspace pair$(\mathcal{M},\bar{\mathcal{M}}^{\perp})$, where $\bar{\mathcal{M}}$ is the closure of $\mathcal{M}$. Here $\bar{\mathcal{M}}^{\perp}:= \{ v \in \RR^p | <u,v> = 0, \forall u \in \bar{\mathcal{M}} \}$.
 $\mathcal{M}$ is the \textbf{model subspace} that typically has a much lower dimension than the original high-dimensional space. $\bar{\mathcal{M}}^{\perp}$ is the \textbf{perturbation subspace} of parameters. For further proofs, we assume the regularization function in ~\eref{eq:ee} is \textbf{decomposable} w.r.t the subspace pair $(\mathcal{M},\bar{\mathcal{M}}^{\perp})$.

\textbf{(C1)} $\mathcal{R}(u+v) = \mathcal{R}(u) + \mathcal{R}(v)$, $\forall u \in \mathcal{M}, \forall v \in \bar{\mathcal{M}}^{\perp}$. 

\cite{negahban2009unified} showed that most regularization norms are decomposable corresponding to a certain subspace pair.
\begin{definition}
\label{def:psi}
\textbf{Subspace Compatibility Constant} \\
Subspace compatibility constant is defined as $\Psi(\mathcal{M},|\cdot|):= \sup\limits_{u \in \mathcal{M}\backslash\{ 0 \}} \frac{\mathcal{R}(u)}{|u|}$ which captures the relative value between the error norm $|\cdot|$ and the regularization function $\mathcal{R}(\cdot)$. 
\end{definition}

For simplicity, we assume there exists a true parameter $\theta^*$ which has the exact structure w.r.t a certain subspace pair. Concretely: 

\textbf{(C2)} $\exists$ a subspace pair $(\mathcal{M},\bar{\mathcal{M}}^{\perp})$ such that the true parameter satisfies $\text{proj}_{\mathcal{M}^{\perp}}(\theta^*) = 0$

Then we have the following theorem.
\begin{theorem}
\label{theo:2}
    Suppose the regularization function in ~\eref{eq:ee} satisfies condition \textbf{(C1)}, the true parameter of ~\eref{eq:ee} satisfies condition \textbf{(C2)}, and $\lambda_n$ satisfies that $\lambda_n \ge \mathcal{R}^*(\hat{\theta}_n - \theta^*)$. Then, the optimal solution $\hat{\theta}$ of ~\eref{eq:ee} satisfies:
    \begin{equation}
        \mathcal{R^*}(\hat{\theta} - \theta^*)\le 2 \lambda_n
    \end{equation}
    \begin{equation}
    \label{eq:theo2:1}
        ||\hat{\theta} - \theta^*||_2 \le 4\lambda_n\Psi(\bar{\mathcal{M}})
    \end{equation}
    \begin{equation}
    \label{eq:theo2:2}
        \mathcal{R}(\hat{\theta} - \theta^*) \le 8\lambda_n\Psi(\bar{\mathcal{M}})^2
    \end{equation}
    
\end{theorem}  

For the proposed JEEK model, $\mathcal{R}(\Omega^{tot}) = ||W_I^{tot}\circ \Omega^{tot}_I||_1 + ||W_S^{tot}\circ \Omega^{tot}_S||_1$. Based on the results in\cite{negahban2009unified}, $\Psi(\bar{\mathcal{M}}) = \sqrt{k_i+k_s}$, where $k_i$ and $k_s$ are the total number of nonzero entries in $ \Omega^{tot}_I$ and $\Omega^{tot}_S$. Using  $\mathcal{R}(\Omega^{tot}) = ||W_I^{tot}\circ \Omega^{tot}_I||_1 + ||W_S^{tot}\circ \Omega^{tot}_S||_1$ in~\rref{theo:2}, we have the following theorem (the same as \rref{theo:jeek}),
\begin{theorem}
Suppose that  $\mathcal{R}(\Omega^{tot}) = ||W_I^{tot}\circ \Omega^{tot}_I||_1 + ||W_S^{tot}\circ \Omega^{tot}_S||_1$ and the true parameter ${\Omega^{tot}}^*$ satisfy the conditions \textbf{(C1)(C2)} and $\lambda_n \ge \mathcal{R}^*(\hat{\Omega}^{tot} - {\Omega^{tot}}^*)$, then the optimal point $\hat{\Omega}^{tot}$ of ~\eref{eq:JEEK} has the following error bounds:
\begin{equation}
\begin{split}
&\max(||W^{tot}_I \circ( \hat{\Omega}^{tot} - {\Omega^{tot}}^*)||_{\infty}, ||W^{tot}_S\circ (\hat{\Omega}^{tot}-{\Omega^{tot}}^*||_{\infty})\\
&\qquad\qquad\qquad\qquad\qquad\qquad\qquad\qquad\qquad\qquad\qquad
\le 2\lambda_n \\
&||\hat{\Omega}^{tot} - {\Omega^{tot}}^*||_{F} \le 4\sqrt{k_i+k_s}\lambda_n \\
&||W^{tot}_I \circ( \hat{\Omega}^{tot}_I - {\Omega^{tot}_I}^*)||_1 + ||W^{tot}_S\circ (\hat{\Omega}^{tot}_S-{\Omega^{tot}_S}^*)||_1\\
&\qquad\qquad\qquad\qquad\qquad\qquad\qquad\qquad
\le 8(k_i+k_s)\lambda_n 
\end{split}
\end{equation}
\label{theo:4}
\end{theorem}

\subsubsection{Proof of \rref{theo:2}}
\label{proof:l2}
\begin{proof}
Let $\delta := \hat{\theta} - \theta^*$ be the error vector that we are interested in.

\begin{equation}
\begin{split}
	\mathcal{R}^*(\hat{\theta}-\theta^*) = \mathcal{R}^*(\hat{\theta}-\hat{\theta}_n+\hat{\theta}_n-\theta^*) \\ \leq \mathcal{R}^*(\hat{\theta}_n-\hat{\theta})+\mathcal{R}^*(\hat{\theta}_n-\theta^*)\leq 2\lambda_n
\end{split}
\end{equation}

By the fact that $\theta^*_{\mathcal{M}^\perp}=0$, and the decomposability of $\mathcal{R}$ with respect to $(\mathcal{M},\mathcal{\bar{M}}^\perp)$

\begin{equation}
\begin{split}	
& \mathcal{R}(\theta^*) \\
& = \mathcal{R}(\theta^*) + \mathcal{R}[\Pi_{\bar{\mathcal{M}}^\perp}(\delta)]- \mathcal{R}[\Pi_{\bar{\mathcal{M}}^\perp}(\delta)] \\
& = \mathcal{R}[\theta^*+\Pi_{\bar{\mathcal{M}}^\perp}(\delta)] - \mathcal{R}[\Pi_{\bar{\mathcal{M}}^\perp}(\delta)] \\
& \leq \mathcal{R}[\theta^* +\Pi_{\bar{\mathcal{M}}^\perp}(\delta) +\Pi_{\bar{\mathcal{M}}}(\delta)] + \mathcal{R}[\Pi_{\bar{\mathcal{M}}}(\delta)] \\ 
&-\mathcal{R}[\Pi_{\bar{\mathcal{M}}^\perp}(\delta)] \\
& = \mathcal{R}[\theta^* + \delta] + \mathcal{R}[\Pi_{\bar{\mathcal{M}}}(\delta)] -\mathcal{R}[\Pi_{\bar{\mathcal{M}}^\perp}(\delta)] 
\end{split}
\label{eq:proof18}
\end{equation}

Here, the inequality holds by the triangle inequality of norm. Since \eref{eq:ee} minimizes $\mathcal{R}(\hat{\theta})$, we have $\mathcal{R}(\theta^*+\Delta) = \mathcal{R}(\hat{\theta}) \leq \mathcal{R}(\theta^*)$. Combining this inequality with \eref{eq:proof18}, we have:

\begin{equation}
	\mathcal{R}[\Pi_{\bar{\mathcal{M}}^\perp}(\delta)] \leq \mathcal{R}[\Pi_{\bar{\mathcal{M}}}(\delta)]
	\label{eq:proof19}
\end{equation}

Moreover, by Hölder's inequality and the decomposability of $\mathcal{R}(\cdot)$, we have:

\begin{equation}
\begin{split}
	& ||\Delta||^2_2 = \langle \delta,\delta \rangle \leq \mathcal{R}^*(\delta)\mathcal{R}(\delta) \leq 2\lambda_n\mathcal{R}(\delta)\\
	& = 2\lambda_n[\mathcal{R}(\Pi_{\bar{\mathcal{M}}}(\delta)) + \mathcal{R}(\Pi_{\bar{\mathcal{M}}^\perp}(\delta))] \leq 4\lambda_n \mathcal{R}(\Pi_{\bar{\mathcal{M}}}(\delta)) \\
	& \leq 4\lambda_n\Psi(\bar{\mathcal{M}})||\Pi_{\bar{\mathcal{M}}}(\delta)||_2
\end{split}
\label{eq:proof20}
\end{equation}{}

where $\Psi(\bar{\mathcal{M}})$ is a simple notation for $\Psi(\bar{\mathcal{M}},||\cdot||_2)$.

Since the projection operator is defined in terms of $||\cdot||_2$ norm, it is non-expansive: 
$|| \Pi_{\bar{\mathcal{M}}}(\Delta)||_2 \leq || \Delta ||_2$. Therefore, by \eref{eq:proof20}, we have:

\begin{equation}
|| \Pi_{\bar{\mathcal{M}}}(\delta)||_2 \leq 4\lambda_n\Psi(\bar{\mathcal{M}}),
\label{eq:proof21}
\end{equation}

and plugging it back to \eref{eq:proof20} yields the error bound 
\eref{eq:theo2:1}.

Finally, \eref{eq:theo2:2} is straightforward from \eref{eq:proof19} and \eref{eq:proof21}.

\begin{equation}
\begin{split}
	& \mathcal{R}(\delta) \leq 2 \mathcal{R}(\Pi_{\bar{\mathcal{M}}}(\delta))\\
	& \leq 2\Psi(\bar{\mathcal{M}})||\Pi_{\bar{\mathcal{M}}}(\delta) ||_2 \leq 8\lambda_n\Psi(\bar{\mathcal{M}})^2.
\end{split}
\end{equation}

\end{proof}

\subsubsection{Conditions of Proving Error Bounds of JEEK}
\label{sec:conditionJEEK}

JEEK achieves similar convergence rates as the SIMULE\cite{wang2017constrained} (W-SIMULE with no additional knowledge) and FASJEM estimator \cite{wang2017fast}. The other multiple sGGMs estimation methods have not provided such convergence rate analysis. 

To derive the statistical error bound of JEEK, we need to assume that $inv(T_v(\hat{\Sigma}^{tot}))$ are well-defined. This is ensured by  
assuming that the true ${\Omega^{(i)}}^*$ satisfy the following conditions  \cite{yang2014elementary}: 

\textbf{(C-MinInf$-\Sigma$):} The true ${\Omega^{(i)}}^*$ \eref{eq:JEEK} have bounded induced operator norm, i.e., $|||{\Omega^{(i)}}^*|||_{\infty} := \sup\limits_{w \ne 0 \in \R^p} \frac{||{\Sigma^{(i)}}^*w||_{\infty}}{||w||_{\infty}} \le \kappa_1 $ .

\textbf{(C-Sparse-$\Sigma$):} The true covariance matrices ${\Sigma^{(i)}}^*$ are ``approximately sparse'' (following \cite{bickel2008covariance}). For some constant $0 \le q < 1$ and $c_0(p)$, $\max\limits_i\sum\limits_{j=1}^p|[{\Sigma^{(i)}}^*]_{ij} |^q \le c_0(p) $. \footnote{This indicates for some positive constant $d$, $[{\Sigma^{(i)}}^*]_{jj} \le d$ for all diagonal entries. Moreover, if $q = 0$, then this condition reduces to ${\Sigma^{(i)}}^*$.} 
 
We additionally require $\inf\limits_{w \ne 0 \in \R^p} \frac{||{\Omega^{(i)}}^* w||_{\infty}}{||w||_{\infty}} \ge \kappa_2$.

\subsubsection{Proof of~\coref{cor:1}}

\begin{proof}
In the following proof, we re-denote the following two notations:
$\Sigma_{tot} := \begin{pmatrix}
  \Sigma^{(1)} & 0 & \cdots & 0 \\
  0 & \Sigma^{(2)} & \cdots & 0 \\
  \vdots  & \vdots  & \ddots & \vdots  \\
  0 & 0 & \cdots & \Sigma^{(K)}
 \end{pmatrix}$

and

$\Omega_{tot} := \begin{pmatrix}
  \Omega^{(1)} & 0 & \cdots & 0 \\
  0 & \Omega^{(2)} & \cdots & 0 \\
  \vdots  & \vdots  & \ddots & \vdots  \\
  0 & 0 & \cdots & \Omega^{(K)}
 \end{pmatrix} $

The condition (C-Sparse$\Sigma$) and condition (C-MinInf$\Sigma$) also hold for $\Omega_{tot}^*$ and $\Sigma_{tot}^*$. In order to utilize~\rref{theo:4} for this specific case, we only need to show that $|| \Omega_{tot}^* - [T_{\nu}(\hat{\Sigma}_{tot})]^{-1}||_{\infty} \leq \lambda_n$ for the setting of $\lambda_n$ in the statement:

\begin{equation}
	\begin{split}
		& || \Omega_{tot}^* - [T_{\nu}(\hat{\Sigma}_{tot})]^{-1}||_{\infty} \\
		& = ||[T_{\nu}(\hat{\Sigma}_{tot})]^{-1}(T_{\nu}(\hat{\Sigma}_{tot})\Omega_{tot}^*-I)||_{\infty} \\
		& \leq ||| [T_{\nu}(\hat{\Sigma}_{tot})w]|||_\infty||T_\nu(\hat{\Sigma}_{tot})\Omega_{tot}^*-I||_\infty \\
		& = |||[T_\nu(\hat{\Sigma}_{tot})]^{-1}|||_\infty||\Omega_{tot}^*(T_\nu(\hat{\Sigma}_{tot})-\Sigma_{tot}^*)||_\infty \\
		& \leq |||[T_\nu(\hat{\Sigma}_{tot})]^{-1}|||_\infty|||\Omega_{tot}^*|||_\infty||T_\nu(\hat{\Sigma}_{tot})-\Sigma_{tot}^*||_\infty.
	\end{split}
	\label{eq:proof2_19}
\end{equation}

We first compute the upper bound of $|||[T_\nu(\hat{\Sigma}_{tot})]^{-1}|||_\infty$. By the selection $\nu$ in the statement, ~\lref{le:1} and~\lref{le:2} hold with probability at least $1-4/p'^{\tau-2}$. Armed with \eref{eq:proof2_18}, we use the triangle inequality of norm and the condition (C-Sparse$\Sigma$): for any $w$,

\begin{equation}
	\begin{split}
		& || T_\nu(\hat{\Sigma}_{tot})w||_\infty = || T_\nu(\hat{\Sigma}_{tot})w -\Sigma w + \Sigma w||_\infty \\
		& \geq || \Sigma w ||_\infty - || (T_\nu(\hat{\Sigma}_{tot})-\Sigma)w||_\infty \\
		& \geq \kappa_2||w||_\infty - || (T_\nu(\hat{\Sigma}_{tot})-\Sigma)w||_\infty \\
		& \geq (\kappa_2 - || (T_\nu(\hat{\Sigma}_{tot})-\Sigma)w||_\infty ) ||w||_\infty
	\end{split}
\end{equation}

Where the second inequality uses the condition (C-Sparse$\Sigma$). Now, by~\lref{le:1} with the selection of $\nu$, we have

\begin{equation}
\label{eq:jeek-proof}
	||| T_\nu(\hat{\Sigma}_{tot}) -\Sigma|||_\infty \leq c_1(\frac{\log (Kp')}{n_{tot}})^{(1-q)/2}c_0(p)
\end{equation}

where $c_1$ is a constant related only on $\tau$ and $\max_i\Sigma_{ii}$. Specifically, it is defined as $6.5(16(\max_i \Sigma_{ii})\sqrt{10\tau})^{1-q}$. Hence, as long as $n_{tot}>(\frac{2c_1c_0(p)}{\kappa_2})^{\frac{2}{1-q}}\log p'$ as stated, so that $||| T_\nu(\hat{\Sigma}_{tot})-\Sigma|||_\infty \leq \frac{\kappa_2}{2}$, we can conclude that $||T_{\nu}(\hat{\Sigma}_{tot})w||_\infty \geq \frac{\kappa_2}{2}||w||_\infty$, which implies $||| [T_\nu(\hat{\Sigma}_{tot})]^{-1}|||_\infty \leq \frac{2}{\kappa_2}$.

The remaining term in \eref{eq:proof2_19} is $||T_\nu(\hat{\Sigma}_{tot})-\Sigma_{tot}^*||_\infty$; $|| T_\nu(\hat{\Sigma}_{tot})-\Sigma_{tot}^*||_\infty \leq || T_\nu(\hat{\Sigma}_{tot})-\hat{\Sigma}_{tot} ||_\infty +||\hat{\Sigma}_{tot} - \Sigma_{tot}^*||_\infty$. By construction of $T_\nu(\cdot)$ in (C-Thresh) and by~\lref{le:2}, we can confirm that $||T_\nu(\hat{\Sigma}_{tot}) - \hat{\Sigma}_{tot} ||_\infty$ as well as $||\hat{\Sigma}_{tot}-\Sigma_{tot}^*||_\infty$ can be upper-bounded by $\nu$.

Therefore, 
\begin{equation}
    \begin{split}
    \label{eq:bound}
        &\max(||W^{tot}_I\circ({\Omega^{tot}}^* - inv(T_v(\hat{\Sigma}^{tot})))||_{\infty},\\ 
        & ||W^{tot}_S\circ({\Omega^{tot}}^* - inv(T_v(\hat{\Sigma}^{tot})))||_{\infty}) \\
        & \le O( \max{\max\limits_{j,k}({W^{tot}_I}_{j,k},{W^{tot}_S}_{j,k})} \sqrt{\frac{\log(Kp)}{n_{tot}}})
    \end{split}
\end{equation}
By combining all together, we can confirm that the selection of $\lambda_n$ satisfies the requirement of~\rref{theo:4}, which completes the proof.
\end{proof}

\section{Appendix: More Background of Proxy Backward mapping and Theorems of $T_v$  Being Invertible }
\label{seca:backward}

The first row of Figure~\ref{fig:digram} summarizes the EE-sGGMs. Two important concepts: 

\paragraph{(1) Backward Mapping: }~The Gaussian distribution is naturally an exponential-family distribution. Based on \cite{wainwright2008graphical},
learning an exponential family distribution from data means to estimate its canonical parameter. For an exponential family distribution, computing the canonical parameter through vanilla graphical model MLE can be expressed as a backward mapping (the first step in Figure~\ref{fig:digram}). For a Gaussian, the backward mapping is easily computable as the inverse of the sample covariance matrix. More details in \sref{subs:bm}.

\paragraph{(2) Proxy Backward Mapping: }~When being high-dimensional, we can not compute the backward mapping of Gaussian through the inverse of the sample covariance matrix. 
Now the key is to find a closed-form and statistically guaranteed estimator as the proxy backward mapping under high-dimensional cases. By the conclusion given by the EE-sGGM, we choose $\{ (\{[T_v(\hat{\Sigma}^{(i)})]^{-1}) \}$ as the proxy backward mapping for $\{ \Omega^{(i)} \}$.

\begin{equation}
\label{eq:Tv}
\begin{split}
[T_v(A)]_{ij}:= \rho_v(A_{ij})
\end{split}
\end{equation}
where $\rho_v(\cdot)$ is chosen to be a soft-thresholding function. 

\subsection{More About Background: backward mapping for an exponential-family distribution:}
\label{subs:bm}

The solution of vanilla graphical model MLE can be expressed
as a backward mapping\cite{wainwright2008graphical} for an exponential family distribution. It estimates the model parameters (canonical parameter $\theta$) from certain (sample) moments. We provide detailed explanations about backward mapping of exponential families,  backward mapping for Gaussian special case and backward mapping for differential network of GGM in this section. 

\paragraph{Backward mapping:} Essentially the vanilla graphical model MLE can be expressed as a backward mapping that computes the model parameters corresponding to some given moments in an exponential family distribution. For instance, in the case of learning GGM with vanilla MLE, the backward mapping is $\hat{\Sigma}^{-1}$ that estimates $\Omega$ from the sample covariance (moment) $\hat{\Sigma}$. 

~Suppose a random variable $X \in \RR^p$ follows the exponential family distribution:
\begin{equation}
\P(X;\theta) = h(X)\text{exp}\{ <\theta, \phi(\theta)> - A(\theta) \}
\label{exp}
\end{equation}
Where $\theta \in \Theta \subset \RR^d$ is the canonical parameter to be estimated and $\Theta$ denotes the parameter space. $\phi(X)$ denotes the sufficient statistics as a feature mapping function $\phi : \RR^p \to \RR^d$, and $A(\theta)$ is the log-partition function. We then define mean parameters $v$ as the expectation of  $\phi(X)$: $v(\theta) := \E[\phi(X)]$, which can be the first and second moments of the sufficient statistics $\phi(X)$ under the exponential family distribution. The set of all possible moments by the moment polytope:
\begin{equation}
\mathcal{M} = \{ v | \exists p \text{ is a distribution s.t. } \E_p[\phi(X)] = v\}
\end{equation}
Mostly, the graphical model inference involves the task of computing moments $v(\theta) \in \mathcal{M}$ given the canonical parameters $\theta \in \encircle{H}$.
We denote this computing as \textbf{forward mapping} :
\begin{equation}
\mathcal{A} : \encircle{H} \to \mathcal{M} 
\end{equation}

The learning/estimation of graphical models involves the task of the reverse computing of the forward mapping, the so-called \textbf{backward mapping} \cite{wainwright2008graphical}. We denote the interior of $\mathcal{M}$ as $\mathcal{M}^0$. \textbf{backward mapping} is defined as:
\begin{equation}
\mathcal{A}^*: \mathcal{M}^0 \to \encircle{H}
\end{equation}
which does not need to be unique. For the exponential family distribution, 
\begin{equation}
	\label{eq:back}
\mathcal{A}^* : v(\theta) \to \theta = \nabla A^*(v(\theta)).
\end{equation}
 Where $A^*(v(\theta)) = \sup\limits_{\theta \in \encircle{H}} <\theta, v(\theta)> - A(\theta)$.

\paragraph{Backward Mapping: Gaussian Case}

If a random variable $X \in \RR^p$ follows the Gaussian Distribution $N(\mu, \Sigma)$, then $\theta = (\Sigma^{-1}\mu, -\frac{1}{2}\Sigma^{-1})$. The sufficient statistics $\phi(X) = (X, XX^T)$, $h(x)=(2\pi)^{-\frac{k}{2}}$, and the log-partition function 
\begin{equation}
A(\theta) = \frac{1}{2}\mu^T\Sigma^{-1}\mu+\frac{1}{2}\log(|\Sigma|)
\label{eq:gauback}
\end{equation}

When performing the inference of Gaussian Graphical Models, it is easy to estimate the mean vector $v(\theta)$, since it equals to $\E[X,XX^T]$. 

When learning the GGM, we estimate its canonical parameter $\theta$  through vanilla MLE. 
Because $\Sigma^{-1}$ is one entry of $\theta$ we can use the backward mapping to estimate $\Sigma^{-1}$.  
\begin{equation} 
  \begin{split}
\theta = (\Sigma^{-1}\mu, -\frac{1}{2}\Sigma^{-1})= \mathcal{A}^*(v)= \nabla A^*(v)&\\ 
= ((\E_{\theta}[XX^T]-\E_{\theta}[X]\E_{\theta}[X]^T)^{-1}\E_{\theta}[X],&\\
 -\frac{1}{2}(\E_{\theta}[XX^T]-\E_{\theta}[X]\E_{\theta}[X]^T)^{-1}).
\end{split}
\end{equation}
By plugging in \eref{eq:gauback} into \eref{eq:back}, we get the backward mapping of $\Omega$ as $(\E_{\theta}[XX^T]-\E_{\theta}[X]\E_{\theta}[X]^T)^{-1}) = \hat{\Sigma}^{-1}$, easily computable from the sample covariance matrix.

\subsection{Theorems of $T_v$  Being Invertible}

Based on \cite{yang2014elementary} for any matrix A, the element wise operator $T_v$ is defined as:
\[
[T_v(A)]_{ij}=
\begin{cases}
 A_{ii} + v  & if\ i=j \\
 sign(A_{ij})(|A_{ij}|-v)& otherwise, i \neq j
\end{cases}
\]

Suppose we apply this operator $T_v$ to the sample covariance matrix $\dfrac{X^{T}X}{n}$ to obtain $T_v(\dfrac{X^{T}X}{n})$. Then, $T_v(\dfrac{X^{T}X}{n})$ under high dimensional settings will be invertible with high probability, under the following conditions:\\
\textbf{Condition-1} ($\Sigma$-Gaussian ensemble) Each row of the design matrix $X \in \RR^{n \times p}$ is i.i.id sampled from $N(0,\Sigma)$.\\
\textbf{Condition-2} The covariance $\Sigma$ of the $\Sigma$-Gaussian ensemble is strictly diagonally dominant: for all row i, $\delta_i := \Sigma_{ii}-\Sigma_{j \neq i} \geq \delta_{min} > 0 $ where $\delta_{min}$ is a large enough constant so that $||\Sigma||\infty \leq \dfrac{1}{\delta_{min}}$. 

This assumption guarantees that the matrix $T_v(\dfrac{X^{T}X}{n})$ is invertible, and its induced $\ell_{\infty}$ norm is well bounded. 
Then the following theorem holds:\\
\begin{theorem}
Suppose Condition-1 and Condition-2 hold. Then for any $v \geq 8(max_i \Sigma_{ii})\sqrt(\dfrac{10\tau\log p'}{n})$, the matrix $T_v(\dfrac{X^{T}X}{n})$ is invertible with probability at least $1-4/{p'}^{\tau-2}$ for $p' := max\{n,p\}$ and any constant $\tau > 2$.
\end{theorem}

Then we provide the error bound of $T_v$ in the first lemma of \sref{sec:proofbm} and use it in deriving the error bound of JEEK.

\subsection{Useful lemma(s) of Error Bounds of (Proxy) Backward Mapping}
\label{sec:proofbm}
\begin{lemma}
\label{le:1}
(Theorem 1 of~\cite{rothman2009generalized}). Let $\delta$ be $\max_{ij}|[\frac{X^TX}{n}]_{ij}-\Sigma_{ij}|$. Suppose that $\nu > 2\delta$. Then, under the conditions (C-Sparse$\Sigma$), and as $\rho_v(\cdot)$ is a soft-threshold function, we can deterministically guarantee that the spectral norm of error is bounded as follows:

\begin{equation}
	||| T_v(\hat{\Sigma}) - \Sigma |||_\infty \leq 5\nu^{1-q}c_0(p)+3\nu^{-q}c_0(p)\delta
	\label{eq:proof2_18}
\end{equation}
\end{lemma}

\begin{lemma}
\label{le:2}
(Lemma 1 of~\cite{ravikumar2011high}). Let $\mathcal{A}$ be the event that

\begin{equation}
	|| \frac{X^TX}{n} - \Sigma ||_\infty \leq  8(\max_i \Sigma_{ii})\sqrt{\frac{10\tau \log p'}{n}}
\end{equation}

where $p' := \max(n,p)$ and $\tau$ is any constant greater than 2. Suppose that the design matrix X is i.i.d. sampled from $\Sigma$-Gaussian ensemble with $n \geq 40\max_i\Sigma_{ii}$. Then, the probability of event $\mathcal{A}$ occurring is at least $1-4/p'^{\tau-2}$.

\end{lemma}

\section{Design $W_S$ and $W_I^{(i)}$: connections with related work and real-world applications}
\label{sec:designW}

In this section, we showcase  with specific examples that our proposed model JEEK can easily incorporate edge-level (like distance)  as well as node-based (like hubs or groups)  knowledge for the joint estimation of multiple graphs. To this end, we introduce four different choices of $W^{tot}_S$ and $W^{tot}_I$ in our formulation~\eref{eq:JEEK}. 
By simply designing different choices of $W^{tot}_S$ and $W^{tot}_I$, we can express different kinds of additional knowledge explicitly without changing the optimization algorithm.

Specifically, we design $W_S$ and $W_I^{(i)}$ for cases like: 
\begin{itemize}
    \item (1)  the additional knowledge is available in the form of a $p*p$ matrix $W$. For instance distance matrix among brain regions in neuroscience study belongs to this type;
    \item  (2) the existing knowledge is not in the form of matrix about nodes. We need to design $W$ for such cases, for example the information of known hub nodes or the information of how nodes fall into groups (e.g., genes belonging to the same pathway or locations). 
\end{itemize}
For the second kind, we showcase three different designs of weight matrices for representing (a) known co-Hub nodes, (b) perturbed hub nodes, and (c) node grouping information. 

The design of knowledge matrices is loosely related to the different structural assumptions used by he JGL 
studies as  (\cite{mohan2013node}, \cite{danaher2013joint}). For example, JGL can use specially designed norms like the one proposed in ~\cite{mohan2013node} to push multiple graphs to have a similar set of nodes as hubs. However JGL can not model additional knowledge like a specific set of nodes are hub nodes (like we know node $j$ is a hub node). Differently, JEEK can design $\{ W_I^{(i)}, W_S \}$ for incorporating such knowledge. Essentially JEEK is complementary to JGL because they capture different type of prior information.

\subsection{Case study I: Knowledge as matrix form like a distance matrix or some known edges}

The first example we consider is exploiting a spatial prior to jointly estimate brain connectivity for multiple subject groups.
Over time, neuroscientists have gathered considerable knowledge regarding the spatial and anatomical information underlying brain connectivity (\textit{i.e.} short edges and certain anatomical regions are more likely to be connected \cite{watts1998collective}). Previous studies enforce these priors via a matrix of weights, $W$, corresponding to edges. To use our proposed model JEEK for such tasks, we can similarly choose $W = W_I^{(i)} = W_S$ in~\eref{eq:JEEK}).

\subsection{Case study II: Knowledge of co-hub nodes}

The structure assumption we consider is graphs with co-hub nodes. Namely, there exists a set of nodes $NId = \{j | j \in \{1,2,\dots,p\}\}$ such that $\Omega^{(i)}_{j,k} \ne 0,\forall i \in \{1,2,\dots,K\}\text{ and } k\in \{1,\dots,p\}$. The above sub-figure of Figure~\ref{fig:hub} is an example of the co-hub nodes. 

A so-called JGL-hub~\cite{mohan2013node} estimator chooses $\mathcal{R}'(\cdot) = \sum\limits_{i<i'}P_q(\Omega^{(i)} - \Omega^{(i')})$ in~\eref{eq:ll} to account for the co-hub structure assumption. Here $P_q(\Theta_1,\Theta_2,\dots, \Theta_k) = 1/2 ||\Theta_1,\dots,\Theta_k||_{\ell_1,\ell_q}$. $\Theta_i$ is a symmetric matrix and $||\cdot||_{\ell_1,\ell_q}$ is the notation of $\ell_1,\ell_q$-norm. JGL-hub formulation needs a complicated ADMM solution with computationally expensive SVD steps.

We design $W_S$ and $W_I^{(i)}$ for the co-hub type knowledge in JEEK via: (1) We initialize $\{ W_I^{(i)}, W_S \}$ with $\mathbf{1}_{p\times p}$; (2) ${W_S}_{j,k} = \frac{1}{\gamma}, \forall j \in NId \text{ and } k \in {1,\dots,p}$  where $\gamma$ is a hyperparameter. Therefore, the smaller weights for the edge connecting to the node $j$ of all the graphs enforce the co-hub structure.; (3). After this process, each entry of $\{ W_I^{(i)}, W_S \}$ equals to either $\frac{1}{\gamma}$ or $1$. The below sub-figure of Figure~\ref{fig:hub} is an example of the designed $W_S$.

\begin{figure}[ht]
    \centering
    \includegraphics[width=\linewidth]{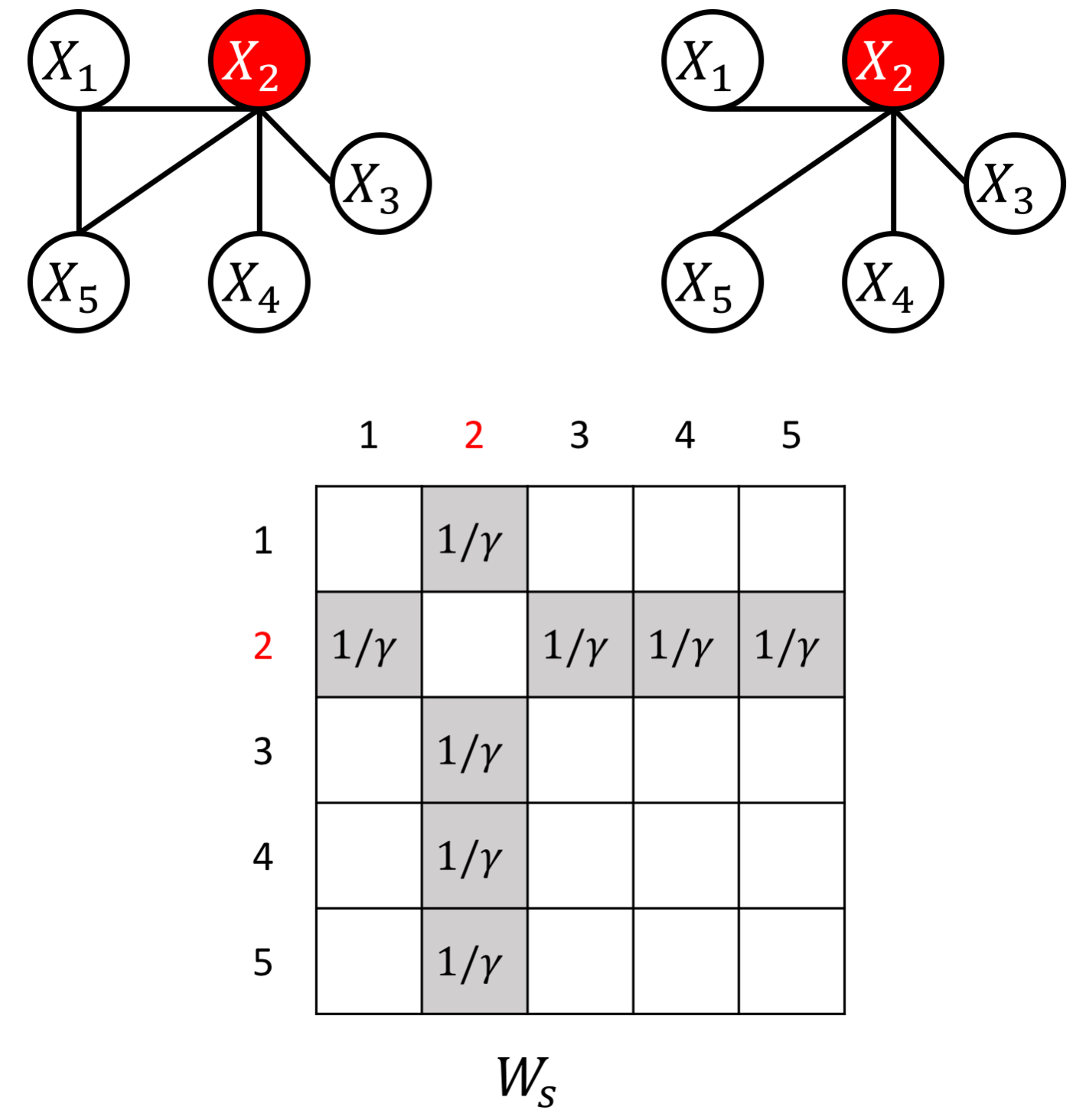}
    \caption{co-hub. Top: An example of the co-hub node structure. Bottom: The designed $W_S$ for the co-hub structure case (white off-diagonal entries are 1).}
    \label{fig:hub}
\end{figure}

\subsection{Case study III: Knowledge of the perturbed hub nodes}

Another structure assumption we study is graphs with perturbed nodes. Namely, there exists a set of nodes $NId = \{j | j \in \{1,2,\dots,p\}\}$ so that there exists $i,i'$ $\Omega^{(i)}_{j,k} \ne 0, \text{ and } \Omega^{(i')}_{j,k} = 0, \forall k\in \{1,\dots,p\}$. The above sub-figure of Figure~\ref{fig:perturb} is an example of the perturbed nodes. 
A so-called JGL-perturb~\cite{mohan2013node} estimator chose $\mathcal{R}'(\cdot) = \sum\limits_{i<i'}P_q((\Omega^{(1)} - \text{diag}(\Omega^{(1)})), \dots ,(\Omega^{(K)} - \text{diag}(\Omega^{(K)})))$ in~\eref{eq:ll}. Here $P_q(\cdot)$ has the same definition as mentioned previously. This JGL-perturb formulation also needs a complicated ADMM solution with computationally expensive SVD steps.

To design $W_S$ and $W_I^{(i)}$ for this type of knowledge in JEEK, we use a similar strategy as the above strategy: (1) We initialize $\{ W_I^{(i)}, W_S \}$ with $\mathbf{1}_{p\times p}$; we let ${W_I^{(i)}}_{j,k} = \frac{1}{\gamma}, {W_I^{(i')}}_{j,k} = \gamma, \forall j \in NId \text{ and } k \in {1,\dots,p}$. Therefore, the different weights for the edge connecting to the node $j$ in different $W_I^{(i)}$ enforce the node-perturbed structure. ; (3). After this process, each entry of $\{ W_I^{(i)}, W_S \}$ equals to either $\frac{1}{\gamma}$,$\gamma$ or $1$. The below sub-figure of Figure~\ref{fig:perturb} is an example of the designed $\{W_I^{(i)}\}$.

\begin{figure}[ht]
    \centering
    \includegraphics[width=\linewidth]{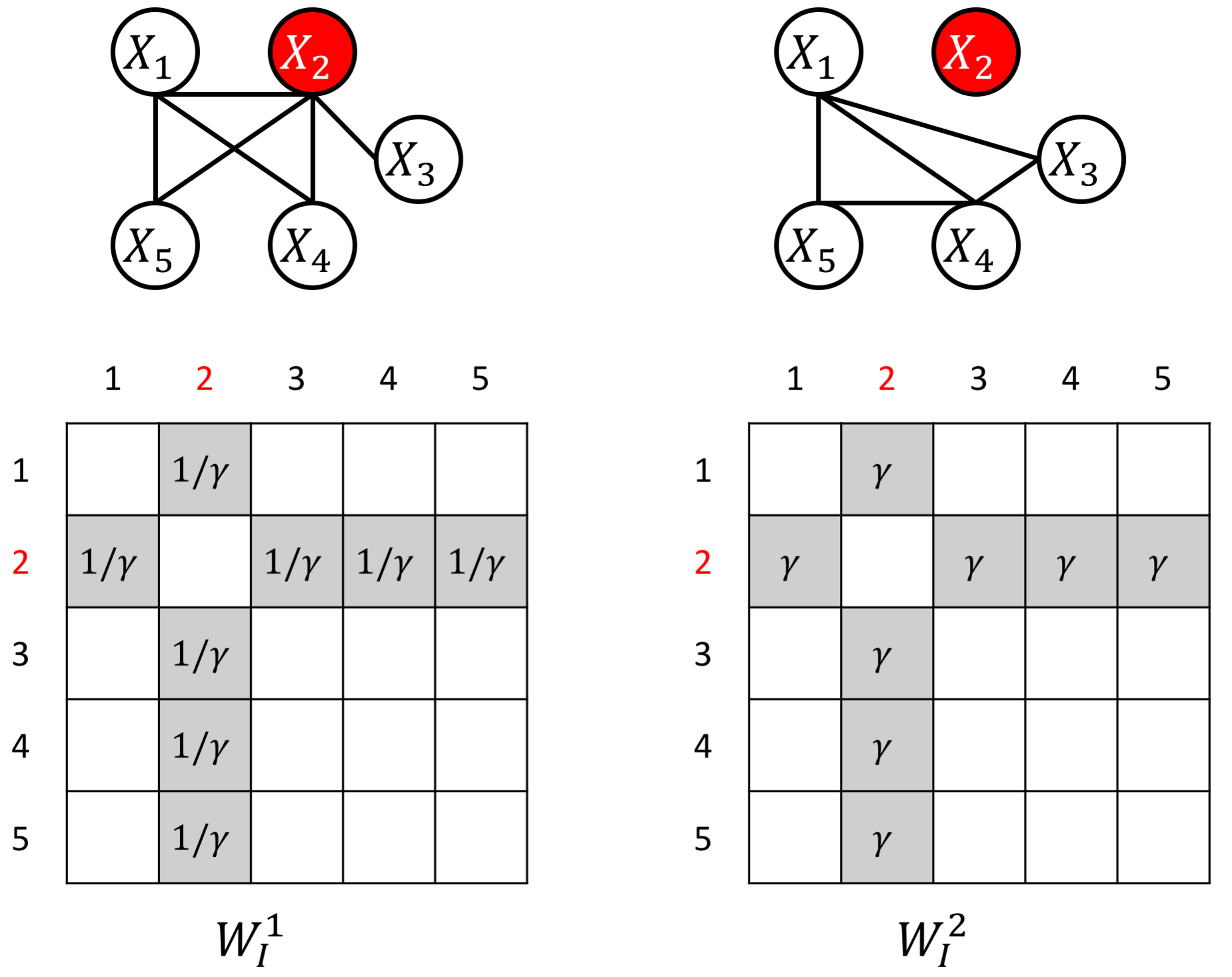}
    \caption{Perturb hub nodes. Top: An example of the perturbed node structure. Bottom: The designed $W_I$ for the perturbed case. (white off-diagonal entries are 1.)}
    \label{fig:perturb}
\end{figure}

\subsection{Case study IV: Knowledge of group information about nodes}

To design $W_S$ and $W_I^{(i)}$ for the group information about a set of nodes, we use a simple three-step strategy: (1) We initialize $\{ W_I^{(i)}, W_S \}$ with $\mathbf{1}_{p\times p}$; (2) We let ${W_S}_{j,k} =  \frac{1}{\gamma}, \forall (j,k) \in Id$ where $\gamma$ is a hyperparameter. Therefore, the smaller weights for the edge $(j,k)$ in all the graphs favors the edges among nodes in the same group. ; (3). After this process, each entry of $\{ W_I^{(i)}, W_S \}$ equals to either $ \frac{1}{\gamma}$ or $1$. The below sub-figure of Figure~\ref{fig:group} is an example of the designed $W_S$ (extra knowledge is that $X_2,X_3,X_4$ belong to the same group).

\begin{figure}[ht]
    \centering
    \includegraphics[width=\linewidth]{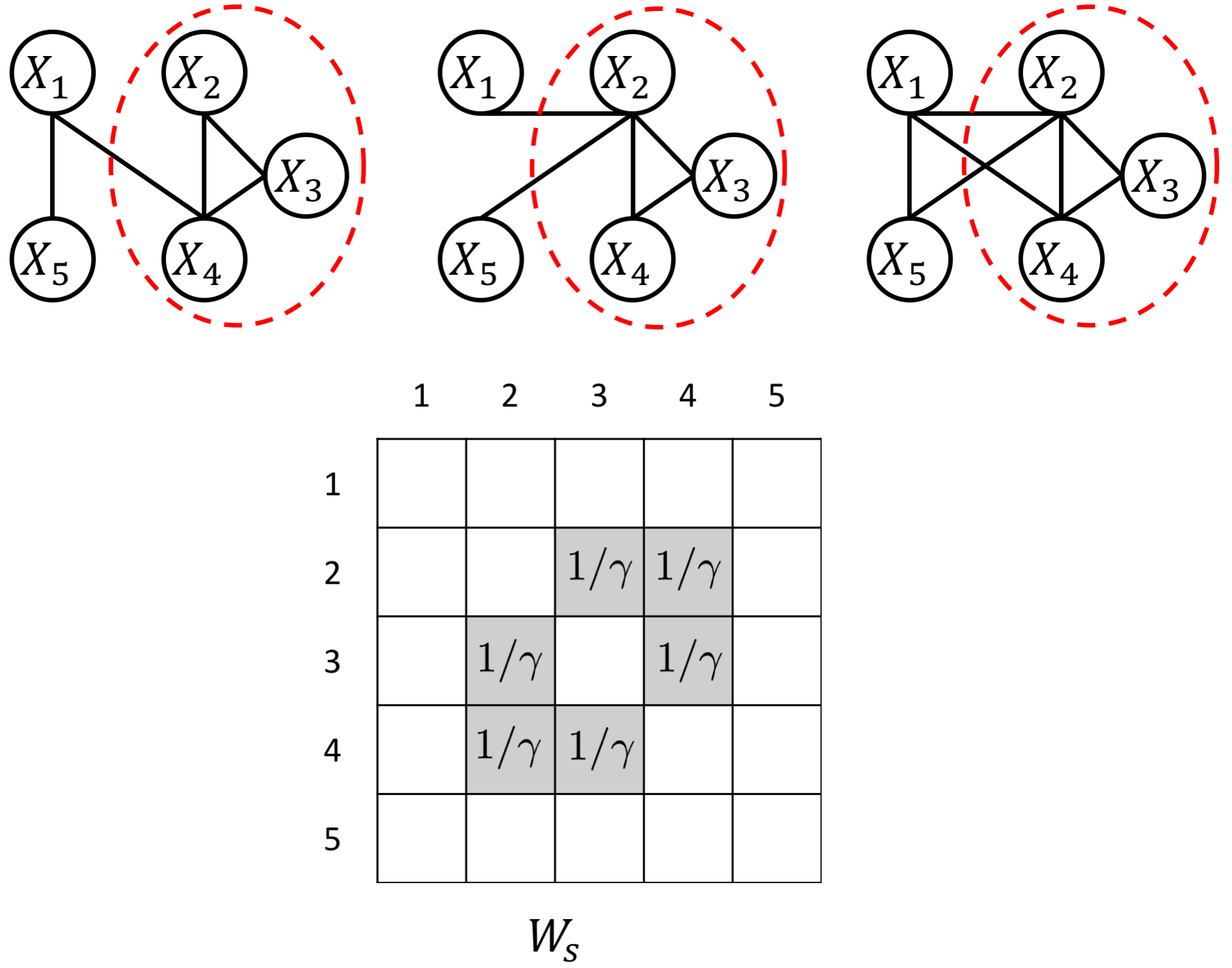}
    \caption{Co-group example case. Top: An example of the co-group node structure. Bottom: The designed $W_S$ for the case. (white off-diagonal entries are 1.)}
    \label{fig:group}
\end{figure}

\section{More about  Experimental Setup}
\label{sec:expS}
\subsection{Experimental Setup}
\label{sec:expset}
\label{sec:expsetMore}

On four types of datasets, we focus on empirically evaluating JEEK with regard to three aspects:  (i) effectiveness, computational speed and scalability in brain connectivity simulation data; (ii) flexibility in incorporating different types of knowledge of known hub nodes in graphs; (iii) effectiveness and computational speed for brain connectivity estimation from real-world fMRI.

\subsection{Evaluation Metrics}
\begin{itemize}
\item{AUC-score:}~The edge-level false positive rate (FPR) and true positive rate (TPR) are used to measure
the difference between the true graphs and the predicted graphs. We obtain FPR vs. TPR curve for each method by tuning  over a range of its regularization parameter. We use the area under the FPR~-TPR curve (AUC-Score) to compare the predicted versus true graph. Here, $\text{FPR} = \dfrac{\text{FP}}
{\text{FP} + \text{TN}}$ and $\text{TPR} = \dfrac{\text{TP}}
{\text{FN} + \text{TP}}$. TP (true positive) and TN (true negative) means the number of true edges and non-edges correctly estimated by the predicted network respectively. FP (false positive) and FN (false negative) are the number of incorrectly predicted nonzero entries and zero entries respectively. 
\item{F1-score:}~ We first use the edge-level F1-score to compare the predicted versus true graph. Here, $\text{F1} = \frac{2\cdot\text{Precision}\cdot\text{Recall}}{\text{Precision} + \text{Recall}}$, where $\text{Precision} = \frac{\text{TP}}{\text{TP} + \text{FP}}$ and $\text{Recall} = \frac{\text{TP}}{\text{TP}+\text{FN}}$. The better method achieves a higher F1-score. 
\item{Time Cost:}~We use the execution time (measured in seconds or log(seconds)) for a method as a measure of its scalability. To ensure a fair comparison, we try $30$ different $\lambda_n$ (or $\lambda_2$)  and measure the total time of execution for each method. The better method uses less time\footnote{The machine that we use for experiments is an AMD 64-core CPU with a 256GB memory.}
\end{itemize}


\paragraph{Evaluations: }~ For the first experiment on brain simulation data, we evaluate JEEK and the baseline methods on F1-score and running time cost. For the second experiment, we use AUC-score and running time cost.\footnote{We cannot use AUC-score for the first set of experiments as the baseline NAK only gives us the best adjacency matrix after tuning over their hyperparameters. It does not provide an option for tuning the $\lambda_n$.} For the third experiment, our evaluation metrics include classification accuracy, likelihood and running time cost.

\begin{itemize}
\item  The first set of experiments evaluates the speed and scalability of our model JEEK on simulation data imitating brain connectivity. We compare both the estimation performance and computational time of JEEK with the baselines in multiple simulated datasets. 
\item In the second experiment, we show JEEK's ability to incorporate knowledge of known hubs in multiple graphs. We also compare the estimation performance and scalability of JEEK with the baselines in multiple simulated datasets.
\item Thirdly, we evaluate the ability to import additional knowledge for enhancing graph estimation in a real world dataset. The dataset used in this experiment is a human brain fMRI dataset with two groups of subjects: autism and control. Our choice of this dataset is motivated by recent literature in neuroscience that has suggested many known weights between different regions in human brain as the additional knowledge.
\end{itemize}

\subsection{Hyper-parameters: }\ We need to tune four hyper-parameters $v$, $\lambda_n$, $\lambda_2$ and $\gamma$: 
\begin{itemize}
    \item $v$ is used for soft-thresholding in JEEK. We choose $v$ from the set $\{ 0.001i|i = 1,2,\dots,1000 \}$ and pick a value that makes $T_v(\hat{\Sigma}^{(i)})$ invertible.
    \item $\lambda_n$ is the main hyper-parameter that controls the sparsity of the estimated network. Based on our convergence rate analysis in Section~\ref{sec:theory}, $\lambda_n \ge C \sqrt{\frac{\log Kp}{n_{tot}}}$ where $n_{tot}=Kn$ and $n=n_i$. Accordingly, we choose 
    $\lambda_n$ from a range of $\{0.01 \times  \sqrt{\frac{\log Kp}{n_{tot}}} \times i| i \in \{ 1,2,3,\dots, 30 \} \}$. 
    \item $\lambda_2$ controls the regularization of the second penalty function in JGL-type estimators. We tune $\lambda_2$ from the set  $\{ 0.01, 0.05, 0.1\}$  for all experiments and pick the one that gives the best results. 
    \item $\gamma$ is a hyperparameter used to design the $W_I^{(i)}, W_S$ (\ref{sec:rel}). The value of $\gamma$  intuitively indicates the confidence of the additional knowledge weights. In the second experiment, we choose $\gamma = \{ 2, 4, 10\}$.
\end{itemize}

\section{More about  Experimental Results}
\label{sec:expR}
\subsection{More Experiment:  Simulate Samples with Known Hubs as Knowledge}
\label{subsec:exp3more}

In this set of experiments, we show empirically JEEK's ability to model knowledge of known hub nodes across multiple sGGMs and its advantages in scalability and effectiveness. We generate multiple simulated Gaussian datasets for both the co-hub and perturbed-hub graph structures.

\paragraph{Simulation Protocol to generate simulated datasets: }~ We generate multiple sets of synthetic multivariate-Gaussian datasets. First, we generate random graphs following the Random Graph Model \cite{rothman2008sparse}. This model assumes $\Omega^{(i)} = \Bb_I^{(i)} + \Bb_S + \delta I$, where each off-diagonal entry in $\Bb^{(i)}$ is generated independently and equals $0.5$ with probability $0.1i$ and 0 with probability $1-0.1i$. The shared part $\Bb_S$ is generated independently and equal to $0.5$ with probability $0.1$ and 0 with probability $0.9$. $\delta$ is selected large enough to guarantee positive definiteness. We generate cohub and perturbed structure simulations, using the following data generation models:
\begin{itemize}
    \item \textbf{Random Graphs with cohub nodes:}  After we generate the random graphs using the aforementioned Random Graph Model, we randomly generate a set of nodes $NId = \{j | j \in \{1,2,\dots,p\}\}$ as the cohub nodes among all the random graphs. The cardinal number of this set equals to $5\% p$. For each of these nodes $j$, we randomly select $90\%$ edges $E_j=\{(j,k)|k \in \{1,2,\dots,p\}\}$ to be included in the graph. Then we set $\Omega^{(i)}_{j,k} = \Omega^{(i)}_{k,j} = 0.5,\forall i \in \{1,2,\dots,K\}\text{ and } (j,k) \in E_j$.
    \item \textbf{Random Graphs with perturbed nodes:}  After we generate the random graphs using the aforementioned Random Graph Model, we randomly generate a set of nodes $NId = \{j | j \in \{1,2,\dots,p\}\}$ as the perturbed hub nodes for the random graphs. The cardinal number of this set equals to $5\% p$.  For all graphs $\{\Omega^{(i)}|\text{i is odd}\}$, for each of these nodes $j \in NId$, we randomly select $90\%$ edges $E_j=\{(j,k)|k \in \{1,2,\dots,p\}\}$ to be included in the graph. We set $\Omega^{(i)}_{j,k} =\Omega^{(i)}_{k,j} =  0.5,\forall \text{ odd } i \in \{1,2,\dots,K\}\text{ and } (j,k)\in E_j$. For all graphs $\{\Omega^{(i)}|\text{i is even}\}$ and nodes $j \in NId$, we randomly select $10\%$ edges $E^{'}_j=\{(j,k)|k \in \{1,2,\dots,p\}\}$ to be included in the graph. We set $\Omega^{(i)}_{j,k} =\Omega^{(i)}_{k,j} =  0.5,\forall \text{ even } i \in \{1,2,\dots,K\}\text{ and } (j,k)\in E^{'}_j$.  This creates a perturbed node structure in the multiple graphs. 
\end{itemize}

\paragraph{Experimental baselines:}
We employ JGL-node for cohub and perturbed hub node structure (JGL-hub and JGL-perturb respectively) and W-SIMULE as the baselines for this set of experiments. The weights in $\{W_I^{tot}, W_S^{tot}\}$ are designed by the strategy mentioned in Section~\ref{sec:designW}.


\paragraph{Experiment Results: }~We assess the performance of JEEK in terms of effectiveness (AUC score) and scalability (computational time cost) through baseline comparison as follows:

\paragraph{(a) Effectiveness:}~We plot the AUC-score for a number of multiple simulated datasets generated by varying the number of features $p$, the number of tasks $K$ and the number of samples $n$. We calculate AUC by varying $\lambda_n$. For the JGL estimator, we additionally vary $\lambda_2$ and select the best AUC (section \ref{sec:expset}). In Figure~\ref{fig:hub} (a) and Figure~\ref{fig:hub} (b), we plot the AUC-Score for the cohub node structure vs varying $p$ and $K$, respectively. Figure~\ref{fig:perturb} (a) and Figure~\ref{fig:perturb} (b) plot the same for the perturbed node structure. In Figure~\ref{fig:hub} (a) and Figure~\ref{fig:perturb} (a), we vary $p$ in the set $\{100,200,300,400,500\}$ and set $K=2$ and $n=p/2$. For $p>300$ and $n=p/2$, W-SIMULE takes more than one month and JGL takes more than one day. Therefore we can not show their results for $p>300$. For both the cohub and perturbed node structures, JEEK consistently achieves better AUC-score than the baseline methods as $p$ is increased. For Figure~\ref{fig:hub}(b) and Figure~\ref{fig:perturb} (b), we vary $K$  in the set $\{2,3,4\}$ and set $p=200$ and $n=p/2$. JEEK consistently has a higher AUC-score than the baselines JGL and W-SIMULE as $K$ is increased.

\begin{figure*}[t]
    \centering
    \includegraphics[width=\linewidth]{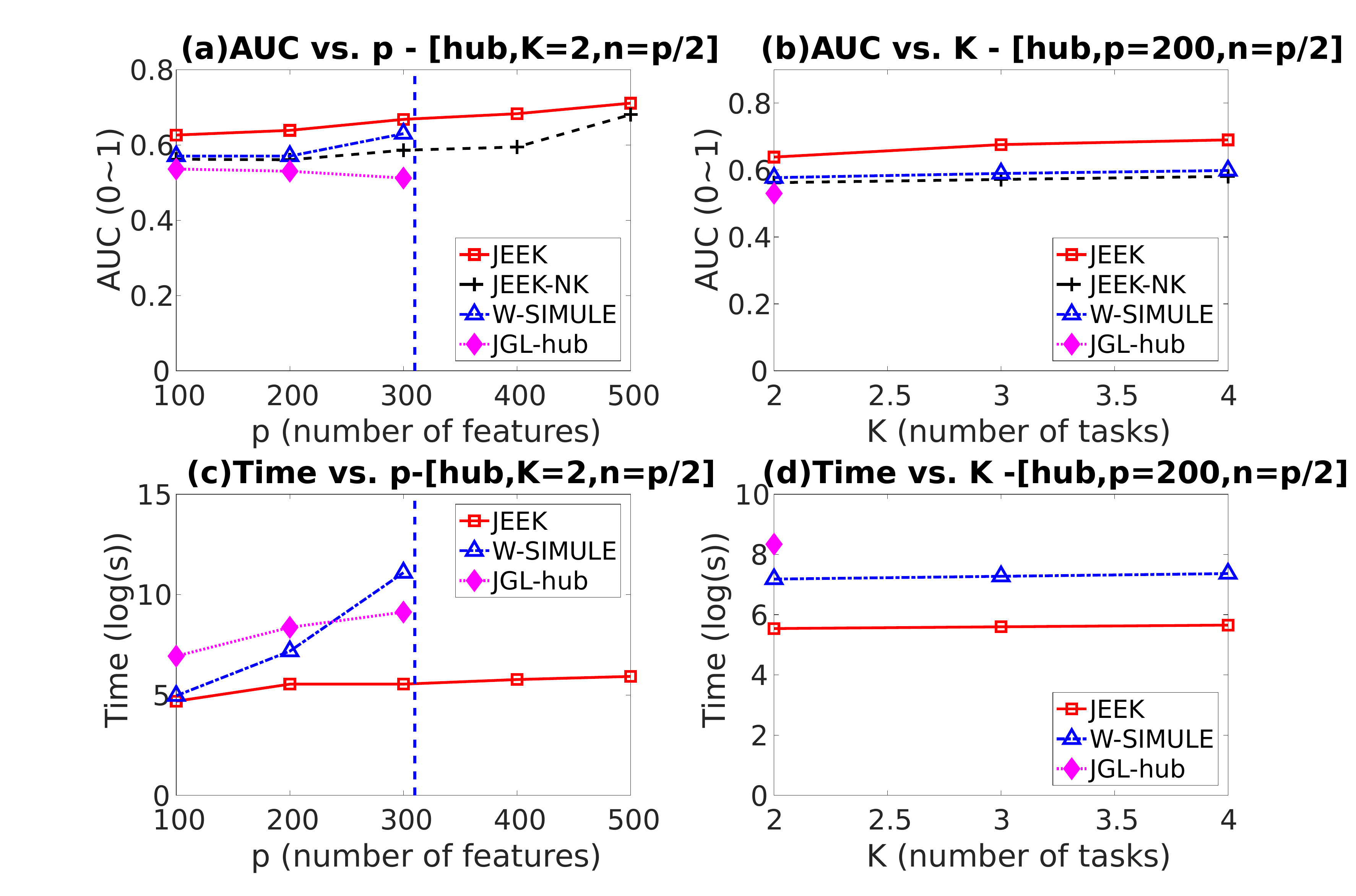}
    \caption{Cohub node structure: (a) AUC-score vs the number of features ($p$). (b) AUC-score vs the number of tasks ($K$). (c) Time cost (log(seconds)) vs the number of features ($p$). (d) Time cost (log(seconds)) vs the number of tasks ($K$). For $p>300$ and $n=p/2$ W-SIMULE takes more than one month and JGL takes more than one day (indicated by dotted blue line). JGL package can only run for $K=2$.}
    \label{fig:hub}
\end{figure*}

\paragraph{(b) Scalability:}~   In Figure~\ref{fig:hub} (c) and (d), we plot the computational time cost for the cohub node structure vs the number of features $p$ and the number of tasks $K$, respectively. Figure~\ref{fig:perturb} (c) and (d) plot the same for the perturbed node structure. We interpolate the points of computation time of each estimator into curves. For each simulation case, the computation time for each estimator is the summation of a method's execution time over all values of $\lambda_n$.  In Figure~\ref{fig:hub}(c) and Figure~\ref{fig:perturb}(c), we vary $p$  in the set $\{100,200,300,400,500\}$ and set $K=2$ and $n=p/2$. When $p>300$ and $n=p/2$, W-SIMULE takes more than one month and JGL takes more than one day. Hence, we have omitted their results for $p>300$. For both the cohub and perturbed node structures, JEEK is consistently more than 5 times faster as $p$ is increased. In Figure~\ref{fig:hub}(d) and Figure~\ref{fig:perturb} (d), we vary $K$ in the set $\{2,3,4\}$ and fix $p=200$ and $n=p/2$. JEEK is $~50$ times faster than the baselines for all cases with $p=200$ and as $K$ is increased. In summary, JEEK is on an average more than $~10$ times faster than all the baselines.

\begin{figure*}[t]
    \centering
    \includegraphics[width=\linewidth]{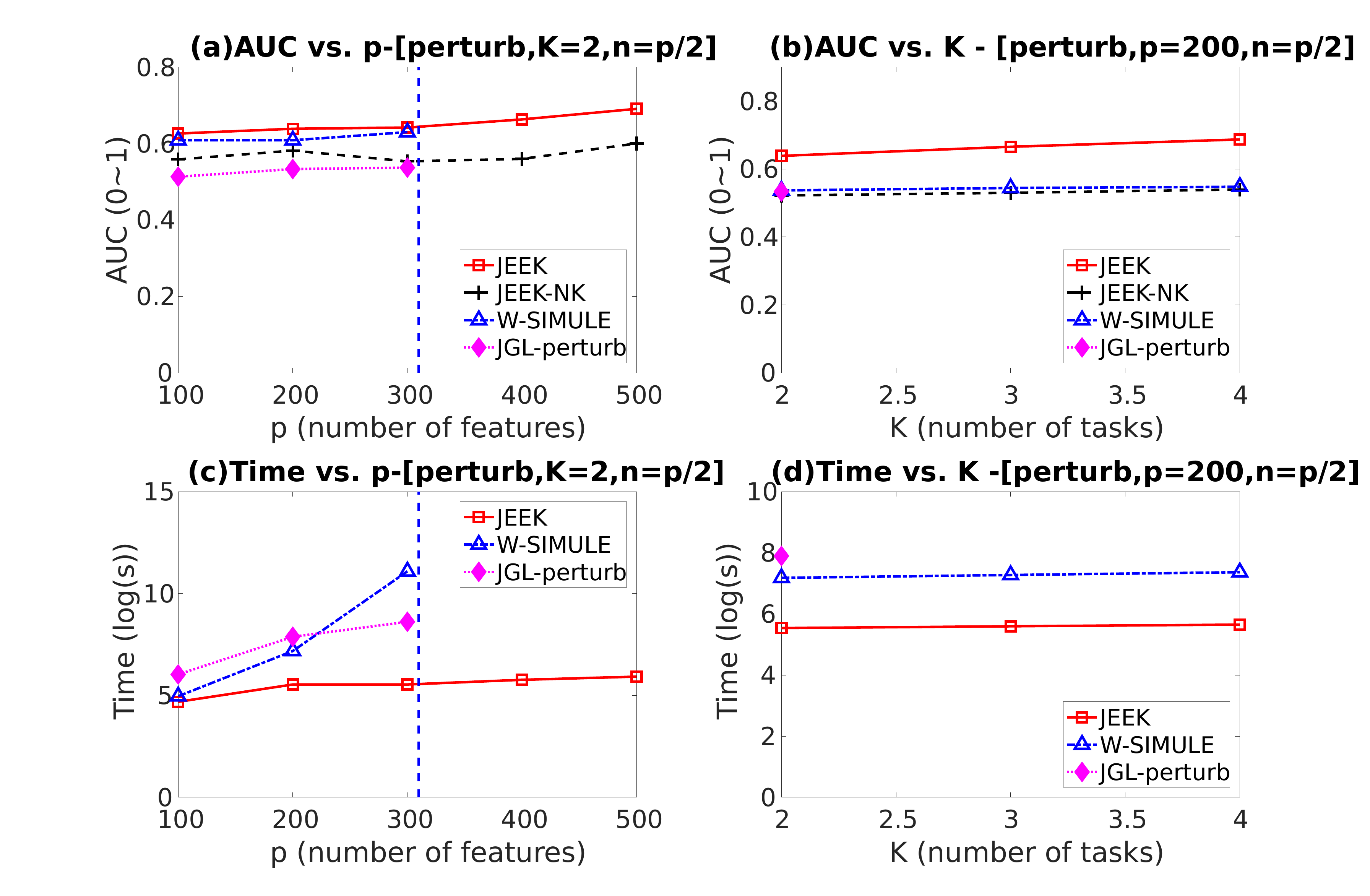}
    \caption{Perturbed node structure: (a) AUC-score vs the number of features ($p$). (b) AUC-score vs the number of tasks ($K$). (c) Time cost (log(seconds)) of JEEK and the baseline methods vs the number of features ($p$). (d) Time cost (log(seconds)) vs the number of tasks ($K$). for $p>300$ and $n=p/2$, W-SIMULE takes more than one month and JGL takes more than one day (indicated by dotted blue line). JGL package can only run for $K=2$.}
    \label{fig:perturb}
\end{figure*}

\paragraph{(c) Stability of Results when varying $W$ matrices:}~Additionally, to account for JEEK's explicit structure assumption, we also vary the ratio of known hub nodes to the total number of hub nodes. The known hub nodes are used to design the $\{W_I^{i},W_S\}$ matrices(details in Section~\ref{sec:rel}). In Figure~\ref{fig:ratio}(a) and (b), AUC for JEEK increases as the ratio of the number of known to total hub nodes increases. The initial increase in AUC is particularly significant as it confirms that JEEK is effective in harvesting additional knowledge for multiple sGGMs. The increase in AUC is particularly significant in the perturbed node case (Figure~\ref{fig:ratio}(b)). The AUC for the hub case does not have a correspondingly large increase with an increase in ratio because the total number of hub nodes are only $5\%$ of the total nodes. In comparison, an increase in this ratio leads to a more significant increase in AUC because the perturbed node assumption has more information than the cohub node structure. We show in Figure \ref{fig:ratio}(c) and (d) that the computational cost is largely unaffected by this ratio for both the cohub and perturbed node structure.

We also empirically check how the  parameter $r$ in the designed knowledge weight matrices influences the performance. In Figure~\ref{fig:gamma}(a) and (b), we show that the designed strategy for including additional knowledge as $W$ is not affected by variations of $\gamma$. We vary $\gamma$ in the set of $\{2,4,10\}$. In summary, the AUC-score(Figure~\ref{fig:gamma}(a),(b)) and computational time cost(Figure~\ref{fig:gamma}(c),(d)) remains relatively unaffected by the changes in $\gamma$ for both co-hub and perturbed-hub case.

\begin{figure*}[t]
    \centering
    \includegraphics[width=\linewidth]{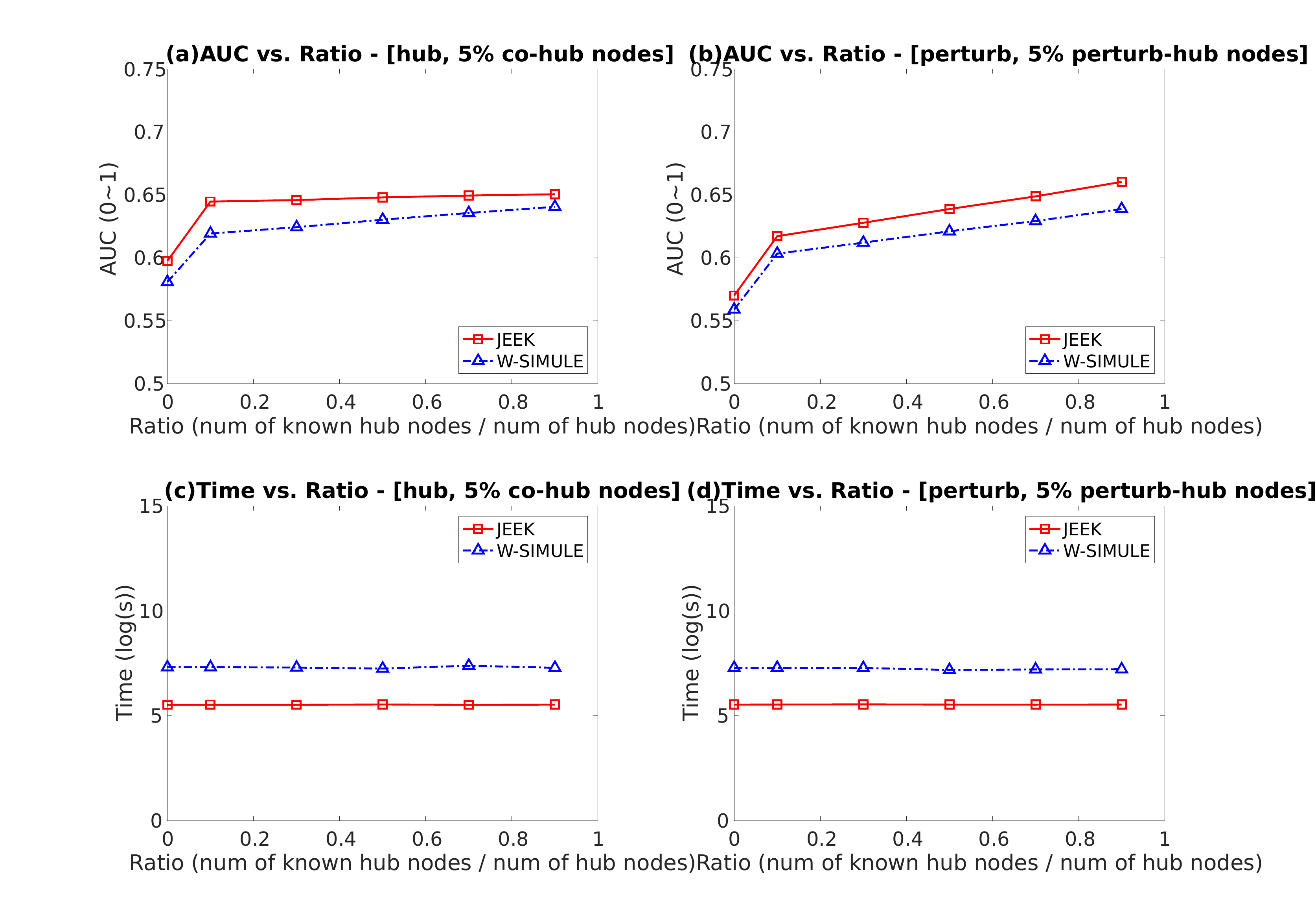}
    \caption{AUC-Score vs. ratio of number of known hub nodes to number of total hub nodes for (a) Cohub node structure (b) perturbed node structure. Computational Time Cost vs. ratio of number of known hub nodes to number of total hub nodes for (a) Cohub node structure (b) perturbed node structure.   }
    \label{fig:ratio}
\end{figure*}

\begin{figure*}[t]
    \centering
    \includegraphics[width=\linewidth]{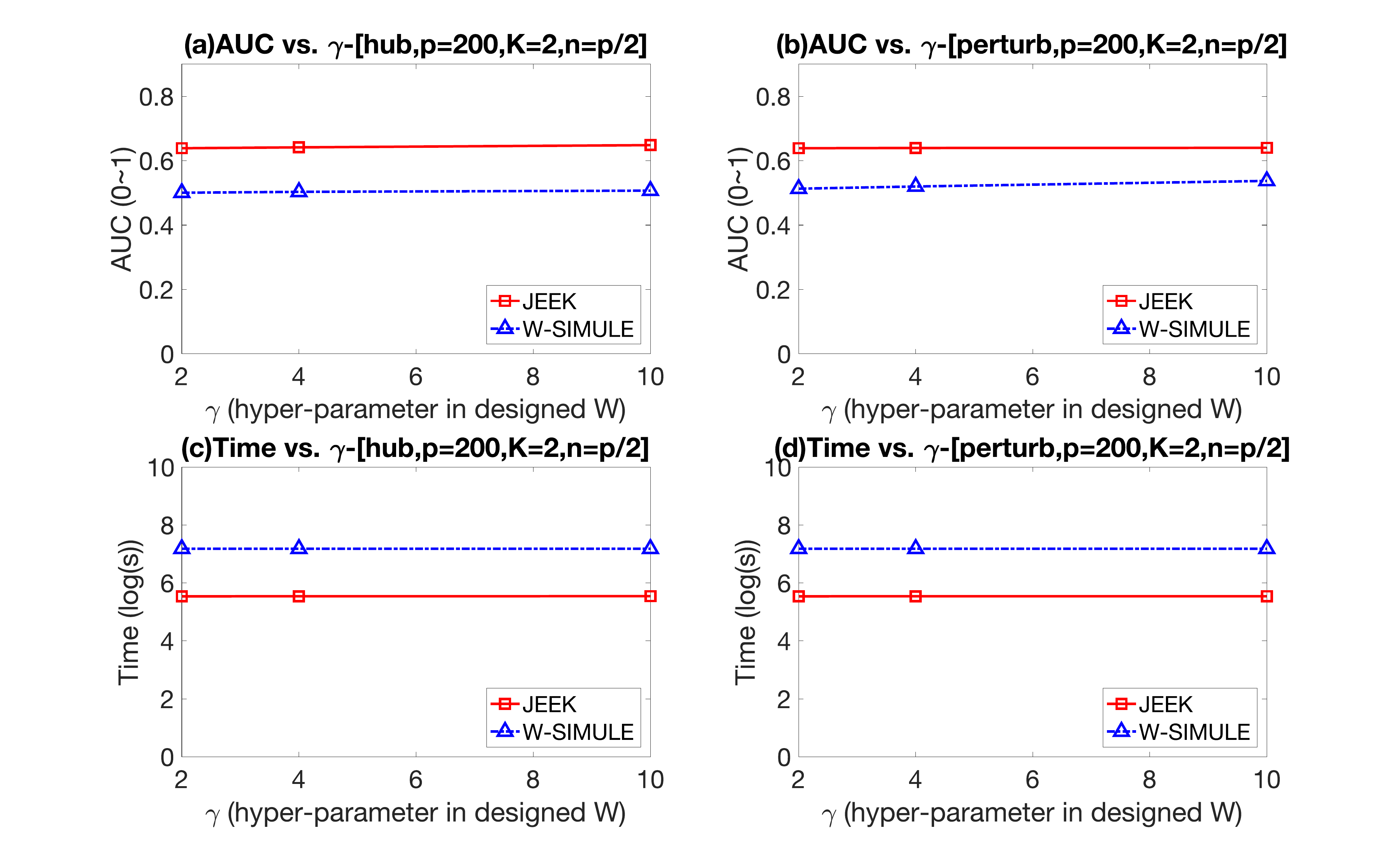}
    \caption{AUC-Score vs. $\gamma$ (a) Cohub node structure for (b) perturbed node structure. Computational Time Cost vs. $\gamma$ for (a) Cohub node structure (b) perturbed node structure.}
    \label{fig:gamma}
\end{figure*}

Figure~\ref{fig:varyn} empirically shows the performance of our methods and baselines when varying the number of samples. We vary $n$ in the set $\{100,200,400\}$ and fix $p=200$ and $K=2$. In Figure~\ref{fig:varyn} (c) and (d), we plot the time cost vs the number of samples $n$ for the cohub and perturbed node structures respectively.  JEEK is much faster than both JGL-node (JGL-hub and JGL-perturb) and W-SIMULE for all cases. Also, the time cost of JEEK does not vary significantly as $n$ increases. In Figure~\ref{fig:varyn} (a) and (b) we also present the AUC-score vs the varying number of samples $n$ for the cohub and perturbed node structures respectively.  For both the cohub and perturbed node structure, JEEK achieves a higher AUC-score compared to W-SIMULE and JGL-node (JGL-hub and JGL-perturb) when $p>n$. The only cases in which the W-SIMULE performs better in Figure~\ref{fig:varyn} (a) and (b) is the low dimensional case ($p=200$, $n=400$). This is as expected because JEEK is designed for high dimensional data situations.

\begin{figure*}[ht]
    \centering
    \includegraphics[width=\linewidth]{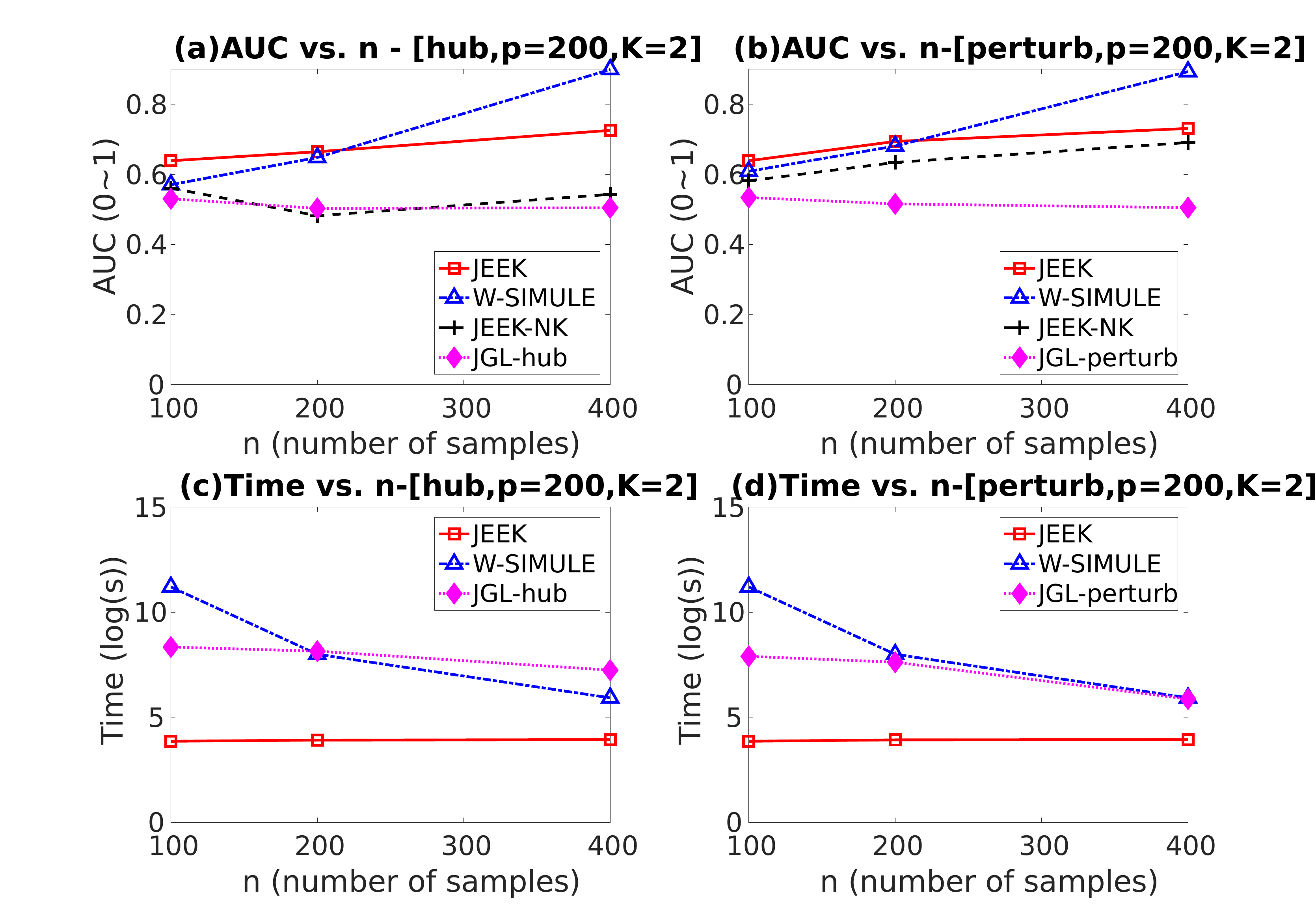}
    \caption{AUC vs. number of samples $n$ for (a) cohub node structure (b) perturbed node structure. Computational Time Cost vs. number of samples for (c) cohub node structure and (d) perturbed node structure.}
    \label{fig:varyn}
\end{figure*}

\subsection{More Experiment: Gene Interaction Network from Real-World Genomics Data}
\label{subsec:exp4more}

Next, we apply JEEK and the baselines on one real-world  biomedical data: gene expression profiles describing  many human samples across multiple cancer types aggregated by \cite{mccall2011gene}.

Advancements in genome-wide monitoring have resulted in enormous amounts of data across most of the common cell contexts, like  multiple common cancer types \cite{cancer2011integrated}.  Complex diseases such as cancer are the result of multiple genetic and epigenetic factors. Thus, recent research has shifted towards the identification of multiple
genes/proteins that interact directly or indirectly in contributing to certain disease(s). Structure learning of sGGMs on such heterogeneous datasets can uncover statistical dependencies among genes and understand how such dependencies vary from normal to abnormal or across different diseases. These structural variations are highly likely  to be contributing  markers that influence or cause the diseases.

Two major cell contexts are selected from the human expression dataset provided by \cite{mccall2011gene}: leukemia cells (including 895 samples and normal blood cells (including 227 samples). Then we choose the top 1000 features from the total 12,704 features (ranked by variance) and perform graph estimation on this two-task dataset. We explore two type of knowledge in the experiments.

The first kind (DAVID) is about the known group information about nodes, such as genes belonging to the same biological pathway or cellular location. We use the popular ``functional enrichment'' analysis tool DAVID  \cite{da2008systematic} to get a set of group information about the 1000 genes. Multiple different types of groups are provided by DAVID and we pick the co-pathway. We only use the grouping information covering 20\% of the nodes (randomly picked from 1000). 
The derived dependency graphs are compared by using the number of predicted edges being validated by three major existing protein/gene interaction databases \cite{prasad2009human,orchard2013mintact,stark2006biogrid}  (average over both  cell contexts).

The second type (PPI) is using existing known edges as the knowledge, like the known protein interaction databases for discovering gene networks (a semi-supervised setting for such estimations). We use three major existing protein/gene interaction databases \cite{prasad2009human,orchard2013mintact,stark2006biogrid}.  We only use the known interaction edge information covering 20\% of the nodes (randomly picked from 1000). 
The derived dependency graphs are compared by using the number of predicted edges that are not part of the known knowledge and are being validated by three major existing protein/gene interaction databases \cite{prasad2009human,orchard2013mintact,stark2006biogrid}  (average over both  cell contexts).

We would like to point out that the interactions JEEK and baselines find represent statistical dependencies between genes that vary across multiple cell types. There exist many possibilities for such interactions, including like physical protein-protein interactions, regulatory gene pairs or signaling relationships. Therefore, we combine multiple existing databases for a joint validation. The numbers of matches between interactions in databases and those edges predicted by each method have been shown as the $y$-axis in Figure~\ref{fig:sim_hubPPI}(c). 
It clearly shows that JEEK  consistently outperforms two baselines.

\subsection{More Experiment: Simulated Samples about Brain Connectivity with Distance as Knowledge}
\label{subsec:exp2more}

In this set of experiments, we confirm JEEK's ability to harvest additional knowledge using brain connectivity simulation data. Following \cite{bu2017integrating}, we employ the known Euclidean distance between brain regions as additional knowledge $W$ to generate simulated datasets. To generate the simulated graphs, we use $p_{j,k}=inv.logit(10-W_{j,k}/3)$ as the probability of an edge between nodes $j$ and $k$ in the graphs, where $W_{j,k}$  is the Euclidean distance between regions $j$ and $k$ of the brain.  

The generate datasets all have $p=116$ corresponding to the number of brain regions in the distance matrix shared by \cite{bu2017integrating}. We vary $K$ from the set $\{2,3,4\}$ with $n=p/2$.  The F1-scores  for JEEK, JEEK-NK and W-SIMULE is the best F1-score after tuning over $\lambda_n$. The hyperparameter tuning for NAK is done by the package itself.

\paragraph{Simulated brain data generation model: }~ We generate multiple sets of synthetic multivariate-Gaussian datasets. To imitate brain connectivity, we use the Euclidean distance between the brain regions as additional knowledge $W$ where $W_{j,k}$ is the Euclidean distance between regions $j$ and $k$. We fix $p=116$ corresponding to the number of brain regions \cite{bu2017integrating}. We generate the graph $\Omega^{(i)}$ following $\Omega^{(i)} = \Bb_I^{(i)} + \Bb_S + \delta I$, where each off-diagonal entry in $\Bb_I^{(i)}$ is generated independently and equals $0.5$ with probability $p_{j,k}=inv.logit(10-W_{j,k}/3)$ and $0$ with probability $1-p_{j,k}$ \cite{bu2017integrating}. Similarly, the shared part $\Bb_S$ is generated independently and equal to $0.5$ with probability $p_{j,k}=inv.logit(10-W_{j,k}/3)$ and $0$ with probability $1-p_{j,k}$. $\delta$ is selected large enough to guarantee the positive definiteness. This choice ensures there are more direct connections between close regions, effectively simulating brain connectivity. For each case of simulated data generation, we generate $K$ blocks of data samples following the distribution $N(0, (\Omega^{(i)})^{-1})$. Details see Section~\ref{sec:expsetMore}.

\paragraph{Experimental baselines:}
We choose W-SIMULE, NAK and JEEK with no additional knowledge(JEEK-NK) as the baselines. (see Section~\ref{sec:rel}).

\paragraph{Experiment Results: }~ We compare JEEK with the baselines regarding two aspects-- (a) Scalability (Computational time cost), and (b) Effectiveness (F1-score). Figure~\ref{fig:sim_brain}(a) and Figure~\ref{fig:sim_brain}(b) respectively show the F1-score vs. computational time cost with varying number of tasks $K$ and the number of samples $n$. In these experiments, $p=116$ corresponding to the number of brain regions in the distance matrix provided by \cite{bu2017integrating}. In Figure~\ref{fig:sim_brain}(a), we vary $K$ in the set $\{2,3,4\}$ with $n=p/2$. In Figure~\ref{fig:sim_brain}(b), we vary $n$ in the set $\{p/2,p,2p\}$ and fix $K=2$. The F1-score plotted for JEEK, JEEK-NK and W-SIMULE is the best F1-score after tuning over $\lambda_n$. The hyperparameter tuning for NAK is done by the package itself. For each simulation case, the computation time for each estimator is the summation of a method's execution time over all values of $\lambda_n$. The points in the top left region of Figure~\ref{fig:sim_brain} indicate higher F1-score and lower computational cost. Clearly, JEEK outperforms its baselines as all JEEK points are in the top left region of Figure~\ref{fig:sim_brain}. JEEK has a consistently higher F1-Score and is almost $6$ times faster than W-SIMULE in the high dimensional case. JEEK performs better than JEEK-NK, confirming the advantage of integrating additional knowledge in graph estimation. While NAK is fast, its F1-Score is nearly $0$ and hence, not useful for multi-sGGM estimation. 

\subsection{More Experiment:  Brain Connectivity Estimation from Real-World fMRI}
\label{subsec:exp1more}

\paragraph{Experimental Baselines:} We choose W-SIMULE as the the baseline in this experiment. We also compare JEEK to JEEK-NK and W-SIMULE-NK to demonstrate the need for additional knowledge in graph estimation. 


\paragraph{ABIDE Dataset:} This data is from the Autism Brain Imaging Data Exchange (ABIDE) \cite{di2014autism}, a publicly available resting-state fMRI dataset. The ABIDE data aims to understand human brain connectivity and how it reflects neural disorders \cite{van2013wu}. The data is retrieved from the Preprocessed Connectomes Project \cite{craddock2014preprocessed}, where preprocessing is performed using the Configurable Pipeline for the Analysis of Connectomes (CPAC) \cite{craddock2013towards} without global signal correction or band-pass filtering. After preprocessing with this pipeline, $871$ individuals remain ($468$ diagnosed with autism). Signals for the 160 (number of features $p=160$) regions of interest (ROIs) in the often-used Dosenbach Atlas \cite{dosenbach2010prediction} are examined.

\paragraph{Distance as Additional Knowledge:} To select the weights $\{ W_I^{(i)},W_S\}$, two separate spatial distance matrices $W$ were derived from the Dosenbach atlas. The first, referred to as \textit{anatomical\textsuperscript i}, gives each ROI one of 40 well-known, anatomic labels (\textit{e.g.} ``basal ganglia", ``thalamus"). Weights $W_{j,k}$ take the low value $i$ if two ROIs have the same label, and the high value $10-i$ otherwise. The second additional knowledge matrix, referred to as \textit{dist\textsuperscript i}, sets the weight of each edge ($W_{j,k}$) to its spatial length, in MNI space\footnote{MNI space is a coordinate system used to refer to analagous points on different brains.}, raised to the power $i$. Then $W_I^{(i)} = W_S = W$.

\paragraph{Cross-validation:} Classification is performed using the 3-fold cross-validation suggested by the literature \cite{poldrack2008guidelines}\cite{varoquaux2010brain}. The subjects are randomly partitioned into three equal sets: a training set, a validation set, and a test set. Each estimator produces $\hat{\Omega}^{(1)}- \hat{\Omega}^{(2)}$ using the training set. Then, these differential networks are used as inputs to linear discriminant analysis (LDA), which is tuned via cross-validation on the validation set. Finally, accuracy is calculated by running LDA on the test set. This classification process aims to assess the ability of an estimator to learn the differential patterns of the connectome structures. We cannot use NAK to perform classification for this task, as NAK outputs only an adjacency matrix, which cannot be used for estimation using LDA.

\paragraph{Parameter variation:}
The results are fairly robust to variations of the $W$. (see Table~\ref{tab:parameter_searching}). The effect of changing $W$ seems to have a fairly small effect on the log-likelihood of the model. This is likely because both penalize picking physically long edges, which agrees with observations from neuroscience. The \textit{dist} $W$ effectively encourages the selection of short edges, and the \textit{anatomical} $W$ also has substantial spatial localization.

\begin{table*}[hbt]
\centering
\caption{Variations of the $W$ and multi-task component yield fairly stable results.}
\label{tab:parameter_searching}

\begin{tabular}{|c|c|c|c|c|}
\hline
\textbf{Prior}                       & \multicolumn{2}{c|}{Sparsity=8\%} & \multicolumn{2}{c|}{Sparsity=16\%} \\ \hline
                                     & Log-Likelihood   & Test Accuracy  & Log-Likelihood   & Test Accuracy   \\ \hline
\textit{No Additional Knowledge}                    & -294.34          & 0.56           & -283.27          & 0.55            \\ \hline
\textit{dist}                        & -289.12          & 0.53           & -285.69          & 0.55            \\ \hline
\textit{dist\textsuperscript2}       & -283.78          & 0.54           & -282.92          & 0.54            \\ \hline
\textit{anatomical\textsuperscript1} & -292.42          & 0.56           & -289.34          & 0.57            \\ \hline
\textit{anatomical\textsuperscript2} & -291.29          & 0.58           & -285.63          & 0.56            \\ \hline
\end{tabular}
\end{table*}

\bibliographystyle{icml2018}
\bibliography{gm1801}

\end{document}